\documentclass[10pt]{article} % For LaTeX2e
\usepackage{microtype}
\usepackage{graphicx}
\usepackage{subfigure}
\usepackage{booktabs} % for professional tables
\usepackage{multirow}
\usepackage{xcolor}
\usepackage{hyperref}
\usepackage{breqn}
\usepackage{amsmath}
\usepackage{amssymb}
\usepackage{mathtools}
\usepackage{amsthm}
\usepackage{pifont}
\newcommand{\cmark}{\ding{51}}%
\newcommand{\xmark}{\ding{55}}%
\usepackage[capitalize,noabbrev]{cleveref}
\usepackage{subcaption,lipsum}
\theoremstyle{plain}
\newtheorem{theorem}{Theorem}[section]
\newtheorem{proposition}[theorem]{Proposition}
\newtheorem{lemma}[theorem]{Lemma}

\theoremstyle{definition}

\newtheorem{assumption}[theorem]{Assumption}
\theoremstyle{remark}

\usepackage[textsize=tiny]{todonotes}

\usepackage[preprint]{tmlr}

\usepackage{amsmath,amsfonts,bm}

\def\eqref#1{equation~\ref{#1}}
\def\1{\bm{1}}

\DeclareMathAlphabet{\mathsfit}{\encodingdefault}{\sfdefault}{m}{sl}
\SetMathAlphabet{\mathsfit}{bold}{\encodingdefault}{\sfdefault}{bx}{n}

\usepackage{hyperref}
\usepackage{url}
\definecolor{mycolor}{rgb}{0.05 0.05, 0.9}

\title{{Adaptive Self-Distillation for Minimizing Client Drift in Heterogeneous Federated Learning \\}}

\author{\name M.Yashwanth \email yashwanthm@iisc.ac.in \\
      \addr Indian Institute of Science
      \AND
      \name Gaurav Kumar Nayak \email gauravkumar.nayak@mfs.iitr.ac.in  \\
      \addr Indian Institute of Technology (IIT) Roorkee
      \AND
      \name Arya Singh \email f20180762g@alumni.bits-pilani.ac.in\\
      \addr  Indian Institute of Science 
      \AND
      \name Yogesh Simmhan \email simmhan@iisc.ac.in  \\
      \addr Indian Institute of Science
      \AND
      \name Anirban Chakraborty \email anirban@iisc.ac.in  \\
      \addr Indian Institute of Science
      }

\def\month{December}  % Insert correct month for camera-ready version
\def\year{2024} % Insert correct year for camera-ready version
\def\openreview{\url{https://openreview.net/forum?id=K58n87DE4s}} % Insert correct link to OpenReview for camera-ready version

\begin{document}

\maketitle

\begin{abstract}
Federated Learning (FL) is a machine learning paradigm that enables clients to jointly train a global model by aggregating the locally trained models without sharing any local training data. In practice, there can often be substantial heterogeneity (e.g., class imbalance) across the local data distributions observed by each of these clients. Under such non-iid label distributions across clients, FL suffers from the `client-drift’ problem where every client drifts to its own local optimum. This results in slower convergence and poor performance of the aggregated model. To address this limitation, we propose a novel regularization technique based on adaptive self-distillation (ASD) for training models on the client side. Our regularization scheme adaptively adjusts to each client's training data based on the global model's prediction entropy and the client-data label distribution. We show in this paper that our proposed regularization (ASD) can be easily integrated atop existing, state-of-the-art FL algorithms, leading to a further boost in the performance of these off-the-shelf methods. We theoretically explain how incorporation of ASD regularizer leads to reduction in client-drift and empirically justify the generalization ability of the trained model. We demonstrate the efficacy of our approach through extensive experiments on multiple real-world benchmarks and show substantial gains in performance when the proposed regularizer is combined with popular FL methods. The link to the code is \url{https://github.com/vcl-iisc/fed-adaptive-self-distillation}.
\end{abstract}

\section{Introduction}

Federated Learning (FL) is a machine learning paradigm where the clients collaboratively learn a shared model under the orchestration of the server without sharing any of their local training data with other clients or the server. Due to the privacy-preserving nature of FL, it has found many applications in smartphones~\citep{47586,ramaswamy2019federated}, the Internet of Things (IoT), healthcare organizations~\citep{rieke2020future,xu2021federated}, where training data is generated at edge devices or from privacy-sensitive domains. As originally introduced in~\citep{mcmahan2017communication}, FL involves model training across an architecture consisting of one server and multiple clients.
In traditional FL, each client securely holds its training data
due to privacy concerns as well as to avoid large communication
overheads while transmitting the same. At the same
time, these clients aim to collaboratively train a generalized
model that can leverage the entirety of the training data disjointly
distributed across clients.
\begin{figure*}[t]
 \centering
 \subfigure[client 1]{\includegraphics[scale=0.30]{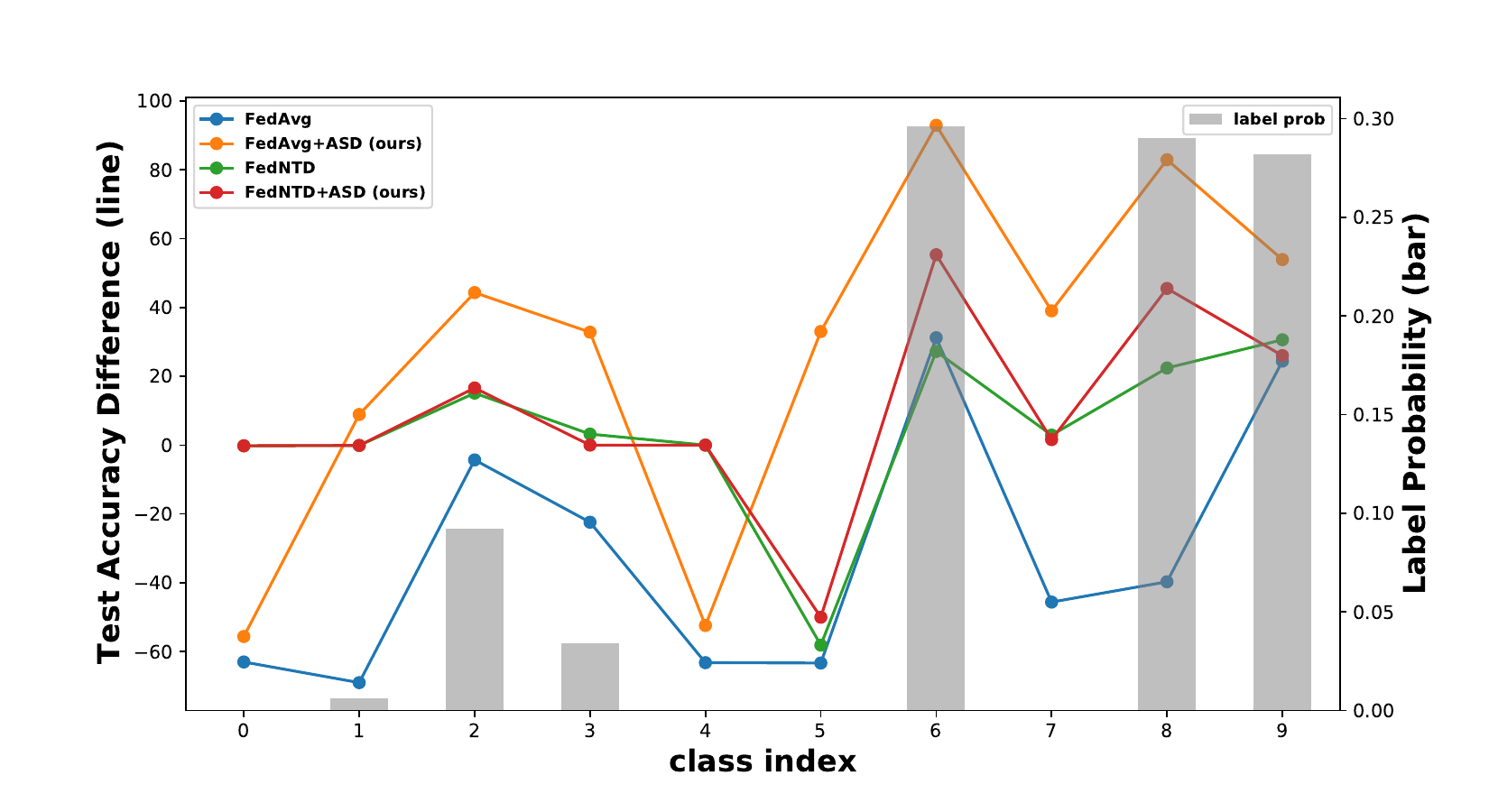}}\hspace{-1.1em}
 \subfigure[client 2]{\includegraphics[scale=0.30]{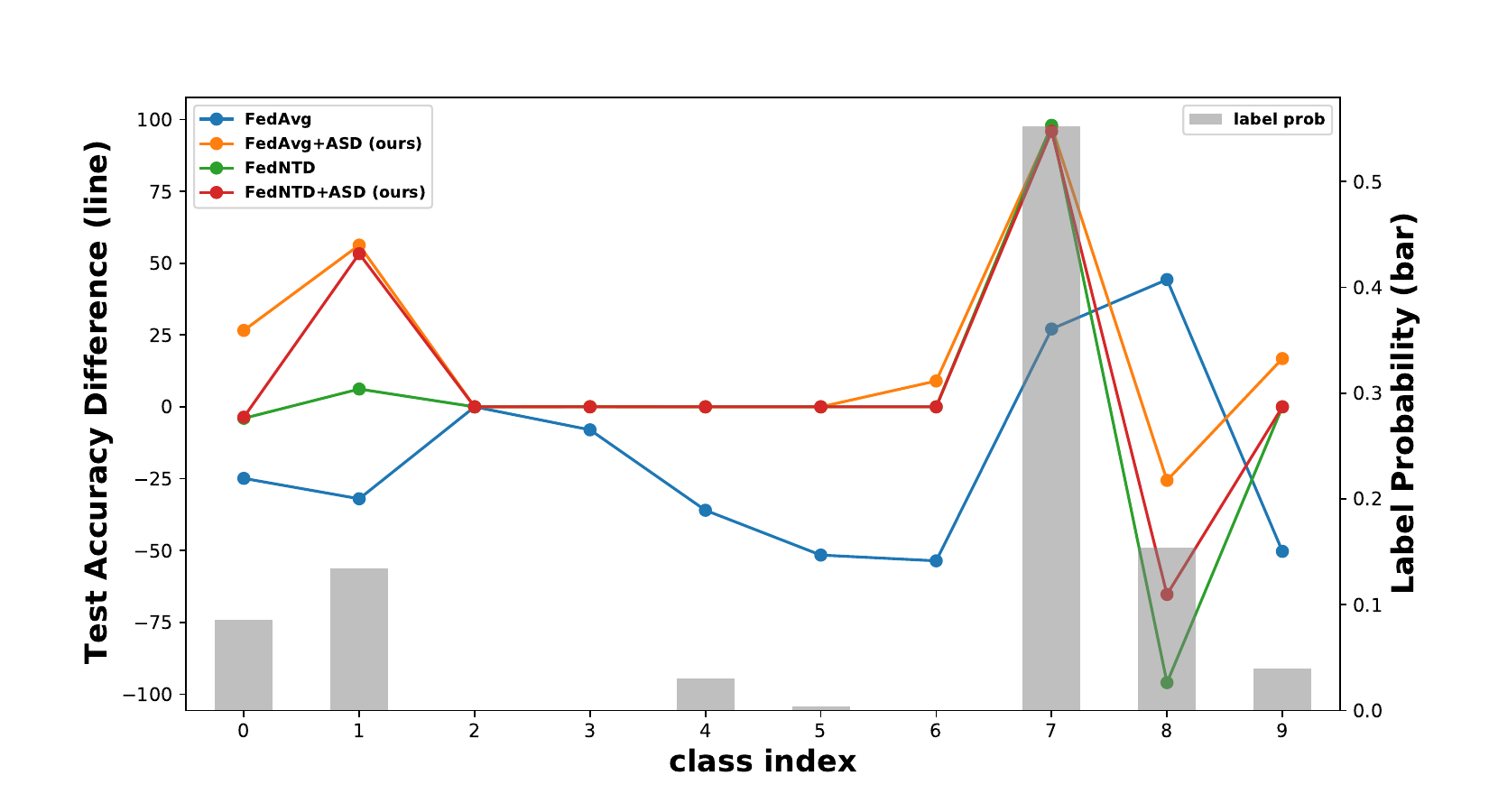}} %\hspace{-1.5em}
 \caption{Impact of one round of local training on the test accuracy of two clients with different label distribution sampled from CIFAR-10 dataset: %The effect of local learning on test accuracy is assessed by plotting the accuracy difference before and after local training. Positive values indicate an improvement in model accuracy post-training. It's noticeable that in cases where the label distribution is sparse (i.e., the probability of occurrence of a label is low), models trained using FedAvg experience a decline in accuracy after local training. However, our proposed adaptive self-distillation regularizer (ASD) when combined with FedAvg i.e FedAvg+ASD it not only preserves knowledge regarding imbalanced classes but also effectively learns from classes with significant representation. The similar effect can be seen with FedNTD and FedNTD+ASD
 The effect of local learning on test accuracy is analyzed by measuring the change in accuracy before and after local training, with positive values indicating improved model performance. Interestingly, in scenarios where classes with low probability of occurrence or under-represented, models trained using FedAvg frequently exhibit a decline in accuracy post-training. In contrast, incorporating our proposed adaptive self-distillation regularizer (ASD) into FedAvg (FedAvg+ASD) not only effectively captures knowledge from well-represented classes but also preserves information about under-represented classes. A similar pattern is observed with FedNTD and FedNTD+ASD.}
%\label{fig:combined_figure}
  \label{fig_intro}
\end{figure*}

\iffalse
\begin{figure}[t]
  %\begin{center}
  \centering
  %\fbox{\rule{0pt}{2in} \rule{.9\linewidth}{0pt}}
  %\includegraphics[width=1.0\linewidth]{fed_distill_latestv.pdf}
  %\includegraphics[width=1.0\linewidth]{plots/Untitled_12_oct-Page-1.drawio (4).pdf}
  %\includegraphics[width=0.7\linewidth]{plots/acc_delta_new.pdf}
  \includegraphics[width=0.70\linewidth]{plots/acc_delta_new_tmp_latest2.pdf}
  %\vspace{-0.2in}
 \caption{Impact of one round of local training on the test accuracy of two clients with CIFAR-10 dataset: The effect of local learning on test accuracy is assessed by plotting the accuracy difference before and after local training. Positive values indicate an improvement in model accuracy post-training. It's noticeable that in cases where the label distribution is sparse (i.e., the probability of occurrence of a label is low), models trained using FedAvg experience a decline in accuracy after local training. However, our proposed adaptive self-distillation regularizer (ASD) when combined with FedAvg i.e FedAvg+ASD not only preserves knowledge regarding imbalanced classes but also effectively learns from classes with significant representation. The similar effect can be seen with FedNTD and FedNTD+ASD.}
  %\end{center}
  \label{fig_intro}
\end{figure}
\fi
%\vspace{-0.25in}
Data ingested at the edge/client devices are often highly heterogeneous as a consequence of the data generation process. They can differ in terms of quantity imbalance (the number of samples at each client are different), label imbalance (empirical label distribution across the clients widely vary), and feature imbalance (features of the data across the clients are non-iid). %FedAvg algorithm is designed to mitigate quantity imbalance. 
When there exists a label or feature imbalance, the objective for every client becomes different as the local minimum for every client objective will be different. In such settings, during the local training, the client's model starts to drift towards its local minimum and farther away from the global objective. This is undesirable as the goal of FL is to converge to a global model that generalizes well across all the clients. This phenomenon, known as `client-drift', is introduced and explored in earlier works~\citep{karimireddy2020scaffold,acar2021federated,wang2021field}. %Due to this, client models tend to forget the knowledge of the global model and overfit to their data. 
In this work, we will be considering only the label heterogeneity. 
\iffalse
One popular way of mitigating such adverse effects of non-iid data distribution owing to label imbalance across clients is via client-side regularization. Here, the client models during local training are explicitly regularized with the global model parameters to minimize client drift. Algorithms such as FedProx~\citep{li2020federated}, SCAFFOLD~\citep{karimireddy2020scaffold}, and FedDyn~\citep{acar2021federated} use regularization at the client-side in the parameter space. 
\fi 
%\textcolor{mycolor}
 {In any given FL round, the client initializes its model with global model weights and then starts training its model using the local data. Due to this, the client training often leads to overfitting the local data and cannot retain the knowledge acquired from the global model in an earlier FL round. %    sTo mitigate catastrophic forgetting, regularization approaches such as~\citep{he2022class,lee2021preservation} are introduced.}

\iffalse
However, they ignore the representations learned by the global model, which can be useful and are explored in distillation-based works in recent literature~\citep{he2022class,lee2021preservation}. 
The primary motivation behind the distillation and regularization works in the context of FL is that the global model will have better representations than the local models.
\fi

Recently~\citep{he2022class} introduced a class-wise adaptive weighting scheme (FedCAD) at the server side. The major drawback of FedCAD is that it assumes the presence of related auxiliary data and reliability on the server to compute the weights for the clients. Dependency on the server for computing the adaptive class-wise weights necessitates the availability of auxiliary data at the server. 
%Another work~\citep{lee2021preservation} proposes FedNTD, which analyses the case of distilling only the incorrect class labels. 
%FedNTD poses the client-drift as the local forgetting problem. As shown in Fig~\ref{fig_intro}, FedNTD is not effective enough to prevent the client-drift as it gives uniform weights to the regularization loss for every sample without considering the label imbalance. Due to this it treats the samples with high and low probability in a similar manner, this makes the client model biased towards the samples with high probability of occurance leading to performance degradation.
Another work~\citep{lee2021preservation} proposes FedNTD which poses the client-drift as a local forgetting problem. It cannot mitigate client-drift effectively since it assigns uniform weights to regularization loss for all samples, independent of label distribution. Consequently, it treats high and low probability samples similarly, biasing the client model towards those with higher probability of occurrence, thus degrading performance.
%The key issue with the existing distillation works is that they cannot explain how the generalization ability of the global model is improved.
To address these issues and motivated by client model regularization in mitigating client drift and to remove the server's dependency on computing client-side weights, we introduce a computationally efficient strategy known as Adaptive Self-Distillation (ASD) for Federated Learning. Importantly, ASD does not require any auxiliary data. We use the KL divergence between the global and local models as the regularizer. For every sample, the weight assigned to the regularization loss is adaptively adjusted based on the global model's prediction entropy and the empirical label distribution of the client's data.
Specifically, when the server model encounters samples with high entropy, we reduce the weighting on the regularization loss, whereas, for samples with a low probability of occurrence, we prioritize the learning from the global model. This adaptive approach enables local models to effectively learn from the cross-entropy loss for more frequent labels while leveraging the global model's guidance for less frequent labels. The adaptive weights are computed without relying on external or proxy data, unlike methods such as FedCAD which relies on external data. Moreover, the additional computational burden on clients is minimal, involving only a single forward pass of the training data. %To the best of our knowledge, this methodology has not been proposed in existing literature.

%We need only a single forward pass to obtain the adaptive weights. 
%We use the KL divergence between the global and local models as the regularizer and the per-sample weights for the regularizer are adaptively adjusted based on the global model entropy and the empirical label distribution of the client's data. We need only a single forward pass to obtain the adaptive weights.
In Fig~\ref{fig_intro}, we explain how the ASD regularization with adaptive weights helps mitigate the client-drift. We analyze the impact of client-drift by observing one round of local training on a particular client model with the CIFAR-10 dataset. We see that FedAvg substantially deteriorates the performance on the labels that have sparse or no representation in the client's local data. After adding the ASD loss, the impact is reduced. The ASD with (adaptive weights) performs the best in terms of local learning and preserving the global model knowledge on the sparse classes.
We theoretically explain the client-drift reduction through our proposed ASD regularizer. In addition, we also provide justification on how ASD leads to improved generalization of the global model. This novel design of our proposed method allows the regularizer to be easily integrated atop any existing FL methods, and this results in substantial performance gains, making it an attractive and compelling solution to the federated learning problem. To the best of our knowledge, this is the first work where the adaptive weights are used for the distillation loss in the FL framework without requiring access to auxiliary data and without the assistance of the server. We would like to clearly point out that the goal of this work is not to directly compete against any particular regularization method used in FL. Our proposed ASD regularizer is of a true plug-and-play nature. With a very negligible computational overhead (discussed in Sec.~\ref{sec_compute}), ASD can be used as an additional regularization on top of any off-the-shelf FL method (either with regularized FL methods such as FedProx, FedDyn or FL methods without regularization such as FedSAM, FedAvg etc.) and further boost their performance across the benchmark datasets such as CIFAR-100/10 and Tiny-ImageNet for both IID and non-IID settings. As a validation, we combine our proposed method
%regularization scheme 
with some of the popular off-the-shelf FL methods such as FedAvg~\citep{mcmahan2017communication}, FedProx~\citep{li2020federated}, FedDyn~\citep{acar2021federated} FedSpeed~\citep{sun2023fedspeed}, FedNTD~\citep{lee2021preservation},
FedSAM~\citep{caldarola2022improving} and 
FedDisco~\citep{ye2023feddisco} and consistently observe performance improvement.

In summary, the key contributions of this work are:
%\vspace{-0.1in}
\begin{itemize}
    \item  We introduced a novel computationally efficient regularization method ASD in the context of Federated Learning %for deep neural networks 
    that alleviates the client drift problem by adaptively weighting the regularization loss for each sample based on the global model's prediction entropy and the label distribution of client data.
  
    \item We demonstrate the efficiency of our method by extensive experiments on datasets such as CIFAR-10, CIFAR-100, and Tiny-ImageNet datasets by combining our proposed ASD regularizer with the popular FL methods and improving their performance.
    %to improve the accuracy and reduce the communication cost by combining our ASD regularizer with the 
    \item  We present a theoretical analysis of the client-drift and show that our regularizer minimizes the client-drift. We also empirically show that ASD promotes better generalization by converging to a flat minimum. 
    %We show the modularity of ASD regularizer by combining with state-of-the-art methods and improving their performance.

    %\item Our proposed approach can be elegantly 
    \end{itemize}

\section{Related Work}
\subsection{Federated Learning (FL)}
In recent times, addressing heterogeneity in Federated Learning has become an active area of research, and the field is developing rapidly. For brevity, we discuss a few related works here. In FedAvg~\citep{mcmahan2017communication}, the two main challenges explored are reducing communication costs~\citep{yadav2016review} and ensuring privacy by avoiding having to share the data. There are some studies based on gradient inversion \citep{geiping2020inverting} raising privacy concerns owing to gradient sharing while some studies have proposed in defense of sharing the gradients~\citep{kairouz2021advances,huang2021evaluating}.
FedAvg is the generalization of local SGD~\citep{stich2018local} by increasing the number of local updates, significantly reducing communication costs for an iid setting, but does not give similar improvements for non-iid data. Several works perform an SGD-type analysis that involves the full device participation, and this breaks the important constraint in FL setup of partial device participation. Some of these attempt to compress the models to reduce the communication cost~\citep{mishchenko2019distributed}. A few works include regularization methods on the client side~\citep{zhu2021data}, and one-shot methods where clients send the condensed data and the server trains on the condensed data~\citep{zhou2020distilled}. 
%Some works propose data sharing strategies to improve performance in the non-iid setting. 
In~\citep{hsu2020federated}, an adaptive weighting scheme is considered on task-specific loss to minimize the learning from samples whose representation is negligible. Flatness-based methods based on SAM called as FedSAM is introduced in~\citep{qu2022generalized,caldarola2022improving}. 

%\textcolor{mycolor}
{
\subsection{Client-Drift in FL}
%Due to data heterogeneity, Federated training suffers from client drift. To mitigate this issue, momentum-based server aggregation is introduced on the server side in~\citep{wangslowmo,hsu2019measuring}. This was later generalized to adapt to any client and server updates in~\citep{reddiadaptive}. FedProx~\citep{li2020federated} introduced a proximal term by penalizing the weights if they are far from the global initialized model. SCAFFOLD~\citep{karimireddy2020scaffold} viewed this problem as one of objective inconsistency and introduced a gradient correction term that acts as a regularizer. Later, FedDyn~\citep{acar2021federated} improved upon this by introducing the dynamic regularization term. In~\citep{kim2024communication}, a proximal term in the client's optimization is introduced based on the accelerated global model. It uses momentum at the server to track the server updates.
Due to data heterogeneity, federated training suffers from client drift. To address this, momentum-based server aggregation was proposed in~\citep{wangslowmo, hsu2019measuring}, which was later extended to handle any client and server updates in~\citep{reddiadaptive} FedProx~\citep{li2020federated} introduced a proximal term to penalize deviations of client weights from the globally initialized model. SCAFFOLD~\citep{karimireddy2020scaffold} tackled the issue as one of objective inconsistency, introducing a gradient correction term as a regularizer. Subsequently, FedDyn~\citep{acar2021federated} enhanced this with a dynamic regularization term. In~\citep{kim2024communication}, a proximal term was introduced in the client's optimization based on the accelerated global model, and momentum was applied on the server to track its updates.
}

%\citep{zhang2020fedpd} attempts to solve this problem by enabling all device participation or none, as described by \citep{acar2021federated}.

%%---------------------------------------------
\subsection{Federated Learning Using Knowledge Distillation}
Knowledge Distillation (KD) introduced by \citep{hinton2015distilling} is a technique to transfer the knowledge from a pre-trained teacher model to the student model by matching the predicted probabilities. Self-distillation was introduced in~\citep{zhang2019your} where the student distills from the same model to the sub-networks of the model. The teacher model predictions are updated every batch. In our method, distillation happens with the full network, and the teacher's predictions are updated after every communication round.   
Adaptive distillation was used in~\citep{tang2019learning}. The server-side KD methods such as FedGen~\citep{seo202216} use KD to train the generator at the server and the generator is broadcasted to the clients in the subsequent round. The clients use the generator to generate the data to provide the inductive bias. This method incurs extra communication of generator parameters along the model and training of the generator in general is difficult. In FedDF~\citep{lin2020ensemble} KD is used at the server that relies on the external data. The KD is performed on an ensemble of client models, especially client models acts as a separate teacher model and then the knowledge is distilled into a single student model (global model).    
In FedNTD~\citep{lee2021preservation} the non-true class logits are used for distillation. This method gives uniform weights to all the samples. 
In FedCAD~\citep{he2022class} and FedSSD~\citep{9826416}, the client-drift problem is posed as a forgetting problem, and a weighting scheme has been proposed. Importantly, the computation of adaptive weights of the client samples is done with the help of the server with the assumption that the server has access to auxiliary data. One shortcoming of this method is the assumption of the availability of auxiliary data on the server, which is impractical.
In ~\citep{zhang2022federated} logits were calibrated based on the label distribution. This is totally different from our approach as we are adjusting the weights of the distillation loss. 
Unlike all of these approaches, we propose a novel ASD strategy that aims to mitigate the challenge of client drift due to non-iid data without relying on the server and access to any form of auxiliary data to compute the adaptive weights.

\section{Method}
We first describe the traditional federated optimization problem, then explain the proposed method of adaptive self-distillation (ASD) in section~\ref{sec:method_asd}. We provide the theoretical and empirical analysis in the sections~\ref{sec:theory_analysis} and~\ref{sec:empirical_analysis} respectively.  
\subsection{Problem Setup}
%We first describe the federated optimization problem in general, in the typical setup of Federated learning. 
We assume there is a single server/cloud and $m$ clients/edge devices. We further assume that client $k$ has its own training dataset $\mathcal{D}_{k}$ with $n_{k}$ training samples drawn iid from the data distribution $\mathbb{P}_{k}(x,y)$. 
The data distributions $\{\mathbb{P}_{k}(x,y)\}_{k=1}^K$ across the clients are assumed to be non-iid.
In this setup, we perform the following optimization. \citep{acar2021federated,mcmahan2017communication}
\begin{equation}
\underset{\mathbf{w}\in \mathbb{R}^d}{\arg\min} \ \left(f(\mathbf{w}) \triangleq \frac{1}{K} \sum_{k\in [K]}f_{k}(\mathbf{w})\right)
\label{eq_f_def}
\end{equation}
where $ f_{k}(\mathbf{w}) $ is the client specific objective function and $\mathbf{w}$ denotes model parameters. The overall FL framework is described in detail in figure~\ref{fig_main}.
\subsection{Adaptive Self-Distillation (ASD) in FL}
\label{sec:method_asd}
 We now describe the proposed method where each client $k$ minimizes the $f_{k}(\mathbf{w})$ as defined below Eq. (\ref{eq1}).
\begin{equation}
f_{k}(\mathbf{w}) \triangleq L_{k}(\mathbf{w}) + \lambda L_{k}^{ASD}(\mathbf{w})
\label{eq1}
\end{equation}
$L_{k}(\mathbf{w})$ is given below.
\begin{equation}
L_k(\mathbf{w}) = \underset{x,y \in \mathbb{P}_k(x,y)}{\mathbb{E}}[l_{k}(\mathbf{w};(x,y))]
\label{exp_erm}
\end{equation} 
Here, $l_{k}$ is cross-entropy loss. The expectation is computed over training samples drawn from $\mathbb{P}_{k}(x,y)$ of a client $k$. This is approximated as the empirical average of the losses corresponding to samples from the Dataset $\mathcal{D}_k$. % and. %The expectation is approximated by taking the empirical average.
\begin{figure*}[t]
  %\begin{center}
  \centering
  %\fbox{\rule{0pt}{2in} \rule{.9\linewidth}{0pt}}
  %\includegraphics[width=1.0\linewidth]{fed_distill_latestv.pdf}
  %\includegraphics[width=1.0\linewidth]{plots/Untitled_12_oct-Page-1.drawio (4).pdf}
  \includegraphics[width=12cm, height=5.5cm]{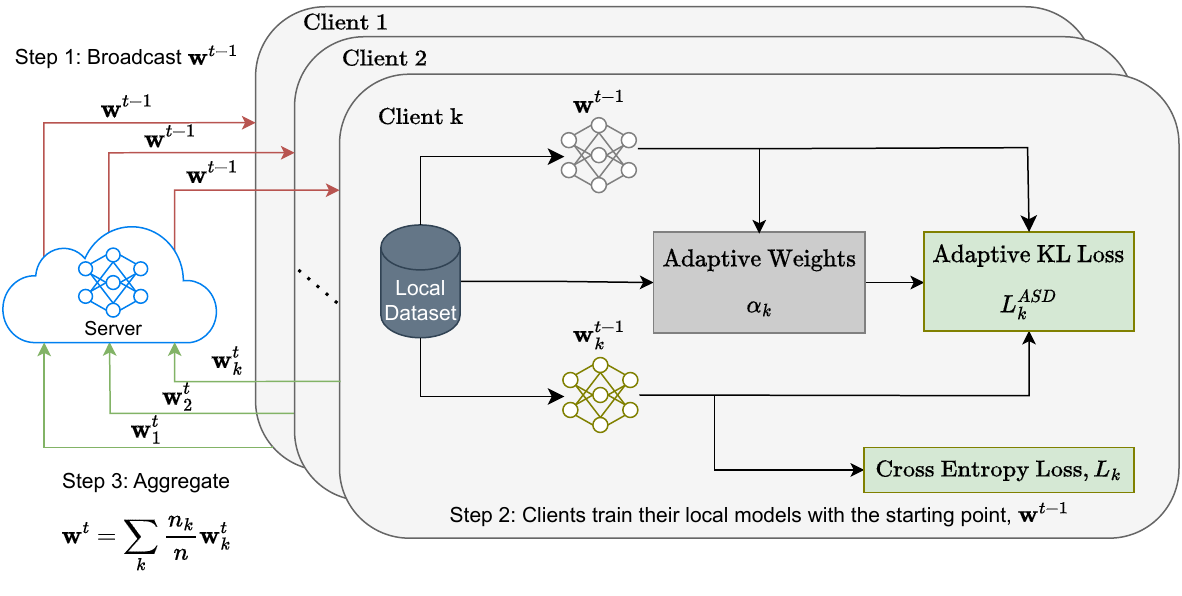}
 \caption{Federated Learning with Adaptive Self-Distillation: The figure describes the overview of the proposed approach based on Adaptive distillation. In \textbf{Step 1}. The server broadcasts the model parameters, In \textbf{Step 2}. clients train their models by minimizing both the cross entropy loss and predicted probability distribution over the classes between the global model and the client model by minimizing the KL divergence, the importance of each sample in the batch is decided by the proposed adaptive scheme as a function of label distribution and the KL term. The server model is fixed while training the client. In \textbf{Step 3}. The server aggregates the client models based on FedAvg aggregation. The process repeats till convergence. }
  %\end{center}
  \label{fig_main}
\end{figure*}
$L_{k}^{ASD}(\mathbf{w})$ in Eq.~\ref{eq1} denotes our proposed Adaptive Self-Distillation loss (ASD) term which considers label imbalance and quantifies how easily the predictions of the local model can drift from the global model. ASD loss is designed so that client models learn from the local data and at the same time not drift too much from the global model. We define (ASD) Loss as follows.
\begin {equation}
L_{k}^{ASD}(\mathbf{w})  \triangleq \mathbb{E}[\alpha_{k}(x,y) \mathcal{D}_{\text{KL}}(q_{g}(x,\mathbf{w}^{t})||q_{k}(x,\mathbf{w}))]
\label{eq7_exp}
\end{equation}
In the above Eq.~\ref{eq7_exp} $\mathbf{w}^t$ represents the global model parameters at FL round $t$ and $\mathbf{w}$ represents the trainable model parameters of client $k$, initialized with $\mathbf{w}^t$ at round $t$. $\alpha_k(x,y)$ denotes the weight for the sample $x$ with label ground truth label $y$. 
For simplicity, we denote  the global model softmax predictions $q_{g}(x,\mathbf{w}^{t})$ as $q_g(x)$ and client model softmax predictions $q_{k}(x,\mathbf{w})$ as $q_k(x)$. $\mathcal{D}_{\text{KL}}$ is the KL divergence.
The Eq.~\ref{eq7_exp} can be approximated by the following equation for a mini-batch. %Note that constant $B$ can be absorbed into $\lambda$.
\begin {equation}
L_{k}^{ASD}(\mathbf{w}) ={\frac{1}{B}} \sum_{i \in [B]}\alpha_{k}(x^i,y^i) \mathcal{D}_{\text{KL}}(q_{g}(x^i)||q_{k}(x^i))
\label{eq7}
\end{equation}
where $B$ is the batch size, $(x^{i},y^i)\in \mathcal{D}_k$, $q_g$ and $q_k$ are  softmax probabilities on the temperature ($\tau$) scaled logits of the global model and client model $k$ respectively. For a class $c$ below Eq.~\ref{eq_pg} and Eq.~\ref{eq_pk} holds.
%\begin{equation}
%q_{g}^{i}[c] = \frac{exp(z_{g}^{i}[c]/T)}{\sum_{m\in C} z_{g}[m]/T}
%\label{eq_pg}
%\end{equation}
\begin{equation}
q_{g}^{c}(x^i) = \frac{exp\left(z_{g}^{c}(x^i)/\tau\right)}{\sum_{m\in C} \exp\left(z_{g}^m(x^i)/\tau\right)}
\label{eq_pg}
\end{equation}

\begin{equation}
q_{k}^c(x^i) = \frac{exp(z_{k}^{c}(x^i)/\tau)}{\sum_{m\in C} \exp\left(z_{k}^m(x^i)/\tau\right)}
\label{eq_pk}
\end{equation}
where $z_{g}(x^i)$, $z_{k}(x^i)$ are the logits predicted on the input $x^i$ by the global model and client model $k$ respectively. The index $i$ denotes the $i^{th}$ sample of the batch.
The $\mathcal{D}_{\text{KL}}(q_{g}(x^i)||q_{k}(x^i))$  is given in Eq. (\ref{eq_kl}).
\begin{equation}
\mathcal{D}_{\text{KL}}(q_{g}(x^i)||q_{k}(x^i)) = \sum_{c = 1}^{C}q_{g}^{c}(x^i)log(q_{g}^{c}(x^i)/q_{k}^{c}(x^i))
\label{eq_kl}
\end{equation}
where $C$ is the number of classes. We use the simplified notation $\alpha_{k}^{i}$ for  distillation weights ${\alpha_{k}}(x^i,y^i)$ and it is given in below Eq.\ref{alpha_eq}.
\begin{equation}
\alpha_{k}^{i}= \frac{\hat{\alpha_{k}}^{i}} {\sum_{i \in B}\hat{\alpha_{k}}^{i}}
\label{alpha_eq}
\end{equation}
and $\hat{\alpha_{k}}^{i}$ is defined as below Eq.~\ref{alpha_hat_eq}
\begin{equation}
\hat{\alpha_{k}}^{i} \triangleq {exp(-\mathcal{H}(x^i)) \over p_{k}^{y^{i}}}
\label{alpha_hat_eq}
\end{equation}

%where $KL_{i}$ is given by Eq. (\ref{eq_kl}). For the $i^{th}$ example, KL divergence between global model and the client model is given by (\ref{eq_kl}).
where $\mathcal{H}(x^i)$ is the entropy of the global model predictions and is given by (\ref{eq_entropy}).
%\footnote{While computing the entropy we set $\tau$ to be $1$ }. 
\begin{equation}
\mathcal{H}(x^i) = \sum_{c = 1}^{C}-q_{g}^{c}(x^i)log(q_{g}^{c}(x^i))
\label{eq_entropy}
\end{equation}

$\mathcal{D}_{\text{KL}}$ in Eq.~\ref{eq_kl} captures how close the local model's predictions are to the global model for any given sample $x^i$. 
Our weighting scheme in Eq.~\ref{eq_entropy} decides how much to learn from the global model for that sample based on the entropy $\mathcal{H}(x^i)$ of the server model predictions and the label distribution of the client data ($p_{k}^{y^{i}}$). $\mathcal{H}(x^i)$ captures the confidence of global model predictions, higher value implies the server predictions are noisy so we tend to reduce the weight, i.e, we give less importance to the global model if its entropy is high. The $p_{k}^{y^{i}}$ is the probability that the sample belong to a particular class. We give more weight to the sample if it belongs to the minority class. This promotes learning from the local data for the classes where the representation is sufficient enough and for the minority classes we encourage them to stay closer to the global model.
 %\textcolor{mycolor}
 {In summary, the choice of alpha is designed to ensure that, when the global model encounters samples with high prediction entropy, we decrease the weighting on the regularization loss. Conversely, for samples with a low probability of occurrence, we prioritize learning from the global model. This adaptive approach enables local models to effectively learn from the cross-entropy loss for more frequent labels while leveraging the global model's guidance for less frequent labels. In Table~\ref{sup:tab:table6}, we highlight the importance of adaptive weights, where we clearly show that ASD with adaptive weights consistently improves performance when combined with off-the-shelf FL methods.}
%In summary, when the server model encounters samples with high entropy, we reduce the weighting on the regularization loss, whereas, for samples with a low probability of occurrence, we prioritize the learning from the global model. This adaptive approach enables local models to effectively learn from the cross-entropy loss for more frequent labels while leveraging the global model's guidance for less frequent labels.  
We approximate the label distribution $p_{k}^{y^i}$ with the empirical label distribution, it is computed as Eq.\ref{emp_pmf}.
\begin {equation}
p_{k}^{y^{i}=c} = \frac{\sum_{i \in |\mathcal{D}_{k}|} \mathbb{I}_{y^{i}=c}}{|\mathcal{D}_{k}|} 
\label{emp_pmf}
\end{equation}
where $\mathbb{I}_{y^{i}=c}$ denotes  the indicator function and its value is $1$ if the  label of the $i^{th}$  training sample belongs to class $c$ else it is $0$. To simplify notation, we use $p_k^c$ for $p_{k}^{y^{i}=c}$ as it depends only on the class $c$ for a client $k$. 
%The $\beta$ and $\beta^{kl}$ are the hyperparameters, $\beta$  captures how much importance will be given to label imbalance, and $\beta^{kl}$ captures the importance given to KL term. 
Finally we use Eq.\ref{emp_pmf} and Eq.~\ref{eq_entropy} to compute the $L_{k}^{ASD}(\mathbf{w})$ defined in Eq.\ref{eq7}.
%%\vspace{-0.1in}
The choice of KL divergence in the Eq.~\ref{eq_kl} is motivated by the seminal work of Hinton et.al., which aims to match the temperature-raised softmax values between the pre-trained teacher model and student model for effective knowledge transfer. 
%\textcolor{mycolor}
{We also analyzed the other statistical divergences such as reverse KL and Jenson-shannon divergence and empirically found that KL divergence is better. More details are presented in the Sec.~\ref{kl_choice} of the Appendix.}

\subsection{Theoretical Analysis of Gradient Dissimilarity}
\label{sec:theory_analysis}
In this section, we perform the theoretical analysis of the client drift. We now introduce the Gradient dissimilarity $G_d$ based on the works of~\citep{li2020federated,lee2021preservation} as a way to measure the extent of client-drift as below.
\begin{equation}
G_d(\mathbf{w},\lambda) = {{{1 \over K }\sum_{k} {\lVert \nabla f_k(\mathbf{w})  \rVert}^2} \over {{\lVert \nabla f(\mathbf{w}) \rVert}^2}}
\label{gd_def}
\end{equation}
$G_d(\mathbf{w},\lambda)$ is function of both the $\mathbf{w}$ and $\lambda$. For convenience, we simply write $G_d$ and mention arguments explicitly when required. $f_k(\mathbf{w})$ in the above Eq.~\ref{gd_def} is same as Eq.~\ref{eq1}.\\
With this, we now establish a series of propositions to show that ASD regularization reduces the Gradient dissimilarity, which as a result, leads to lower client drift. 

%\begin{proposition}
%The minimum value gradient diversity $G_d$ is $1$ ($G_d \geq 1$). It is attained %if all $f_k$ are identical.
%\end{proposition}

\begin{proposition}
 $\inf_{\mathbf{w}\in \mathbb{R}^d} {G_d(\mathbf{w},\lambda)}$ is $1$, $\forall$ $\lambda$
\end{proposition}

The above proposition implies that if all the client's gradients are progressing in the same direction, which means there is no drift $G_d = 1$. The result follows from Jensen's inequality. %The proof is provided in section $4$ of supplementary material. 
The lower value of $G_d$ is desirable and ideally $1$.
To analyze the $G_d$, we need $\nabla f_k(\mathbf{w})$ which is given in the below proposition. 
\begin{proposition}
When the class conditional distribution across the clients is identical, i.e., $\mathbb{P}_{k}(x \mid y) = \mathbb{P}(x \mid y)$  then $\nabla{f_{k}(\mathbf{w})} = \sum_c{p_k^c}(\mathbf{g}_{c} + \lambda  {\gamma}_{k}^{c} \tilde{\mathbf{g}}_c )$, where  $\mathbf{g}_c = \nabla{\mathbb{E}[{l(\mathbf{w};x,y)}\mid{y=c}]}$, $\tilde{\mathbf{g}}_c = \nabla{\mathbb{E}[{ \exp({-\mathcal{H}(x)}) \mathcal{D}_{\text{KL}}(q_{g}(x)||q_{k}(x))} \mid{y=c} ]}$ and ${\gamma}_{k}^{c} = \frac{1}{p_k^c}$. 
\end{proposition}

The result follows from the tower property of expectation and the assumption that class conditional distribution is the same for all the clients. From the above proposition, we can see that the gradients $\nabla{f_{k}(\mathbf{w})}$ only differ due to $p_k^{c}$ which captures the data heterogeneity due to label imbalance. The proof is given in Sec.~\ref{sup:proof} of the appendix.

\begin{assumption}
Class-wise gradients are weakly correlated and similar magnitude  $\mathbf{g}_c^\intercal \mathbf{g}_c \ll \mathbf{g}_c^\intercal \mathbf{g}_m$, $\tilde{\mathbf{g}}_c^\intercal \tilde{\mathbf{g}}_c \ll \tilde{\mathbf{g}}_c^\intercal \tilde{\mathbf{g}}_m $ 
and  %${\mathbf{g}}_c^\intercal \tilde{\mathbf{g}}_c \ll  {\mathbf{g}}_c^\intercal \tilde{\mathbf{g}}_m$ 
for $c \neq m$ 
%and $\sum\limits_{k=1,c1=1,c2=1}^{K,C,C}{p}_{k}^{c1}\mathbf{g}_{c1}^\intercal \tilde{\mathbf{g}}_{c2} < 0$ 
\label{assump1}
\end{assumption}
%%\vspace{-0.15in}
The assumption on weakly correlated class-wise gradients intuitively implies that gradients of loss for a specific class cannot give any significant information on the gradients of the other class. %Finally, the inequality implies that the inter-class gradients between cross-entropy loss gradients and the ASD loss gradients are negatively correlated. %(This is reasonable as the ASD loss gradients are dominated by the classes that have relatively fewer representations in the cross-entropy loss.)  
%On the contrary, If there is a strong correlation among the class-wise gradients, then only few-class gradients should be good enough to train the model.

\begin{proposition}
When the class-conditional distribution across the clients is the same, and the Assumption~\ref{assump1} holds then $\exists$ a range of values for $\lambda$ such that whenever $\lambda \geq \lambda_{c}$ we have $\frac{dG_d}{d\lambda} < 0$ and $G_d(\mathbf{w},\lambda) < G_d(\mathbf{w},0)$.
\label{th1}
\end{proposition}

The proposition implies that there is a value of  $\lambda \geq \lambda_{c}$ such that the derivative of $G_d$ w.r.t $\lambda$ is negative. The proof is given in Sec.~\ref{sup:proof} of the appendix. This indicates that by appropriately selecting the value of $\lambda$ we can make the $G_d$ lower which in turn reduces the client drift. %This key result allows the ASD regularizer to combine with the existing methods and improve their performance. 
%We further discuss the impact of Gradient diversity on convergence in the appendix. 
%To understand the convergence, 
One of the key assumptions on the heterogeneity is the existence of the below quantity.

\begin{equation}
B^2(\lambda) \coloneqq \sup_{\mathbf{w}\in \mathbb{R}^d} {G_d(\mathbf{w},\lambda)}
\label{B_def}
\end{equation}
which leads to the following assumption
%\iffalse
\begin{assumption}
${{{1 \over K }\sum_{k} {\lVert \nabla f_k(\mathbf{w})  \rVert}^2} \leq  B^2(\lambda){{\lVert \nabla f(\mathbf{w}) \rVert}^2}}$
\label{assump2}
\end{assumption}
%%\vspace{-0.1in}
This is the bounded gradient dissimilarity assumption used in~\citep{li2020federated}.
In the following proposition, we show the existence of $\lambda$ such that $B^2(\lambda) < B^2(0)$, which means that with regularizer we can tightly bound the gradient dissimilarity compared to the case without the regularizer i.e., ($\lambda = 0$).   

\begin{proposition}
Suppose the functions $f_k$ satisfy Assumption~\ref{assump2} above  %${\mathbf{g}}_c^\intercal \tilde{\mathbf{g}}_c > 0$ 
then we have  $B^2(\lambda) < B^2(0)$.
\label{label_prop4}
\end{proposition}

\begin{proof}
From~\ref{B_def} we have
\begin{equation}
B^2(\lambda) = \sup_{\mathbf{w}\in \mathbb{R}^d} {G_d(\mathbf{w},\lambda)}
\label{B_def1}
\end{equation}
 For a fixed $\lambda$ as per proposition~\ref{th1} we have the following. 
\begin{equation}
\sup_{\mathbf{w}\in \mathbb{R}^d}{G_d(\mathbf{w},\lambda)} < \sup_{\mathbf{w}\in \mathbb{R}^d}{G_d(\mathbf{w},0)} 
\label{sup_eq}
\end{equation}
The above inequality~\ref{sup_eq} is true as  proposition~\ref{th1} guarantees that the value of $G_d({\mathbf{w}},\lambda) < G_d({\mathbf{w}},0)$ for all $\mathbf{w}$ when $\lambda \geq \lambda_c$. If inequality~\ref{sup_eq} is not true, one can find a $\mathbf{w}$ that contradicts the proposition~\ref{th1} which is impossible.  This means for some value of $\lambda \geq \lambda_{c}$ we have $B^2(\lambda) < B^2(0)$ from Eq.~\ref{B_def} and Eq.~\ref{sup_eq}.
\end{proof}
The key takeaway from the analysis is that by introducing the regularizer we can tightly bound the heterogeneity when compared to the case without the regularizer.
\textit{Based on the works~\citep{karimireddy2020scaffold,li2020federated} we explain that lower $B^2(\lambda)$ implies better convergence, which is also supported by empirical evidence. These details are provided in the Sec.~\ref{sup:convg_disc} of the Appendix.} 
\subsection{Discussison on the Generalization of ASD}
\label{sec:empirical_analysis}
\begin{table}[htp]
\centering 
\caption{The table shows the impact of ASD on the algorithms on CIFAR-100 Dataset using the non-iid partition of $\delta = 0.3$. We consistently see that the top eigenvalue and the trace of the Hessian of the loss of the global model decrease and the accuracy improves when ASD is used. This suggests that by using ASD we can make global model reach a flat minimum towards better generalization.}
\scalebox{0.8}{
\begin{tabular}{c|c|c|c}
\toprule
Algorithm      & Top Eigenvalue $\downarrow$ & Trace $\downarrow$  & Accuracy $\uparrow$  \\ \midrule
FedAvg         & 53.6            & 8516  & 38.67    \\ 
FedAvg + ASD   & \textbf{12.3}            & \textbf{2269}  & \textbf{42.77}    \\ \midrule

FedDyn         & 49.4            & 6675  & 47.56    \\ 
FedDyn + ASD   & \textbf{14.2}            & \textbf{2241}  & \textbf{49.03}    \\\midrule
FedSpeed       & 51.9            & 6937  & 47.39    \\ 
FedSpeed + ASD & \textbf{14.6}            & \textbf{2063}  & \textbf{49.16}    \\ \bottomrule
\end{tabular}
}
\label{main:hess_tab}
\end{table}
%\vspace{-0.1in}
\begin{figure}[htp]
  \centering
   \includegraphics[scale=0.35]{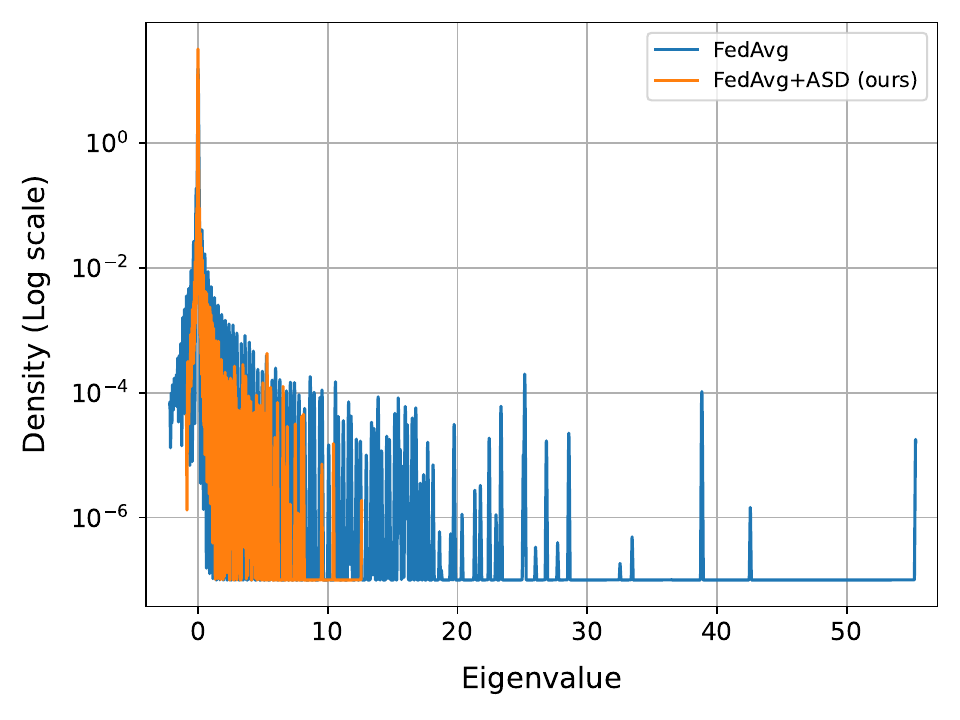}
  %%\vspace{-0.1in}
  \caption{Eigen spectrum with and without the ASD regularizer. It is evident that  ASD regularizer not only minimizes the top eigenvalue but most of the eigenvalues and attains the flatness.}
  \label{fig_eigen_dens}
\end{figure}
The key reason for better generalization is the adaptive self-distillation loss. It has been shown in~\citep{mobahi2020self} that self-distillation improves the generalization in centralized settings. It's been empirically shown in~\citep{zhang2019your} that self-distillation helps the model to converge to flat-minimum. Generally, converging to flat minima is indicative of improved generalization, a concept explored in prior studies such as~\citep{keskar2017on} and~\citep{yao2020pyhessian}. %Lower values of the top eigenvalue and trace serve as typical indicators of the presence of flat minima. 
The top eigenvalue and the trace of the Hessian computed from the training loss are typical measures of `flatness' of the minimum to which the training converges, i.e., lower values of these measures indicate the presence of a flat minimum.
To gain a deeper understanding of this phenomenon in a federated learning setting, we analyzed the top eigenvalue and the trace of the Hessian of the cross-entropy loss for global models obtained with and without the ASD regularizer. %It can be seen why the client models converging to flat minima encourage the global model to converge to a flat minimum by the following argument.
The following argument establishes that if the client models converge to flat minima, it would also ensure convergence of the resultant global model to a flat minimum.
We assume the Hessians of the functions $f_k$ ($k^{th}$ client's local objective), $f$ (resultant global objective) exist and are continuous almost everywhere. Since $f = \frac{1}{K}\sum_{i=1}^{K}{f_i} $, we have $\mathbf{H}(f) = \frac{1}{K} \sum_{i=1}^{K}\mathbf{H}({f_i})$ ($\mathbf{H}(g)$ denotes the Hessian of function $g$). This implies $\mu_{1}({\mathbf{H}(f)}) \leq \frac{1}{K} {\sum_{i = 1}^K } \mu_{1}({\mathbf{H}(f_i)})$ ($\mu_{1}(\mathbf{A})$ denotes top eigenvalue of matrix $\mathbf{A}$). Thus when the local models converge to a flat minimum, it will ensure the convergence of the global model to a flat minimum.
Following the method of~\citep{yao2020pyhessian}, we computed the top eigenvalue and trace of the Hessian. 
In Table~\ref{main:hess_tab}, we observe that FedAvg+ASD attains lower values for the top eigenvalue and trace compared to FedAvg, suggesting convergence to flat minimum. The Eigen density plot in the figure~\ref{fig_eigen_dens} also confirms the same. We use the CIFAR-100 dataset with  non-iid data partitioning of $\delta = 0.3$ (refer to Sec.~\ref{main:exp}). In Table~\ref{main:hess_tab} we have presented our analysis when ASD is combined with FedAvg, FedDyn and FedSpeed. The results for other algorithms are presented in Sec.~\ref{sup:hess_analysis} of appendix. 
A similar concept has been explored in FedSAM~\citep{qu2022generalized,caldarola2022improving}; the issue with SAM-based methods is they require an extra forward and backward pass, which doubles the computational cost on the resource constrained edge devices. However, our method can be applied to SAM-based methods and further improve its performance. ASD consistently attains the flatness with the other FL algorithms and enhances their generalization.
%these are described in Sec.~\ref{sup:hess_anlysis} of the appendix.
\vspace{-0.1in}
%In Table ~\ref{sup:hess_tab}, we analyze the top eigenvalue and the trace of the Hessian for all the algorithms. We observe that when ASD is used with the algorithms, it consistently improves the accuracy and reduces the top eigenvalue and the trace of the Hessian, this implies that ASD helps in converging to the flat minima which is a typical indicator of better generalization. We consider the CIFAR-100 non-iid partition of $\delta=0.3$.   
\section{Experiments}
\label{main:exp}
We perform the experiments on CIFAR-10, CIFAR-100~\citep{krizhevsky2009learning}, Tiny-ImageNet \citep{le2015tiny} datasets with different degrees of heterogeneity in the balanced settings (i.e., the same number of samples per client but the class label distribution of each varies). We set the total number of clients to $100$ in all our %settings. 
experiments. We set the client participation rate to $0.1$, i.e., 10 percent of clients are sampled on an average per communication round, similar to the protocol followed in \citep{acar2021federated}. We build our experiments using %based on the 
publicly available codebase by \citep{acar2021federated}.% for all of our results.
\iffalse
\begin{figure}[htp]
  \centering
   \includegraphics[scale=0.4]{plots/CIFAR100_label_dist.png}
  %%\vspace{-0.3in}
  \caption{Label distribution of clients: A subset containing 10 clients out of 100, and their corresponding label distribution based on Dirichlet distribution is plotted. % label distribution  . We can clearly see the 
  It is easy to observe that the labels are not uniformly distributed across the clients.}
  \label{fig_label_dist}
\end{figure}
\fi 
For generating non-iid data, Dirichlet distribution is used. To simulate the effect of label imbalance, for every client we sample the 'probability distribution' over the classes from the aforementioned Dirichlet distribution $p_{k}^{dir} = Dir(\delta,C)$. Every sample of $p_{k}^{Dir}$ is a vector of length $C$ and all the elements of this vector are non-negative and sum to 1. This vector represents the label distribution for the client. The parameter $\delta$ known as the 'concentration parameter', captures the degree of label heterogeneity. Lower values of $\delta$ capture high heterogeneity and as the value of $\delta$ increases, the label distribution becomes more uniform. Another parameter of Dirichlet distribution (i.e., $C$), its value can be interpreted from the training dataset ( $C=100$ for CIFAR-100). For notational convenience, we omit $C$ from $Dir(\delta,C)$ by simply re-writing as $Dir(\delta)$. By configuring the concentration parameter $\delta$ to 0.6 and 0.3, we sample the data using the Dirichlet distribution across the labels for each client from moderate to high heterogeneity by controlling $\delta$. This is in line with the approach followed %is the same as 
in \citep{acar2021federated} and \citep{yurochkin2019bayesian}. 

\begin{table*}[htp]
\caption{Comparison of Accuracy(\%): We show the accuracy attained by the algorithms across the datasets (CIFAR-100/Tiny-ImageNet) at the end of 500 communication rounds. It can be seen that by combining the proposed approach the performance of all the algorithms can be significantly improved. %FedDyn+ASD attains the best performance when compared to other methods.
}
\centering
%\resizebox{\textwidth}{!} {
\scalebox{0.8} {
\begin{tabular}{l|ccc|ccc}
\toprule
\multirow{3}{*}{Algorithm} %\\\cline{1-2}    
%\\\hline
& \multicolumn{3}{c}{CIFAR-100}                                                                 & \multicolumn{3}{c}{TinyImageNet}                       
\\ \cline{2-7} 
%& \multicolumn{2}{c}{}                                              
%& \multicolumn{1}{c}{}
%& \multicolumn{1}{c|}{\multirow{2}{*}{IID}}
%& \multicolumn{2}{c}{}                                              
%& \multicolumn{1}{c}{}
%& \multicolumn{1}{c}{\multirow{2}{*}{IID}} 
\\ %\cline{2-3} \cline{5-6}
& \multicolumn{1}{c}{$\delta= 0.3$}          
& \multicolumn{1}{c}{$\delta= 0.6$}          
& \multicolumn{1}{c}{IID}                     
& \multicolumn{1}{c}{$\delta= 0.3$}          
& \multicolumn{1}{c}{$\delta= 0.6$}          
& \multicolumn{1}{c}{IID}                     
\\ 
\midrule
FedAvg~\citep{mcmahan2017communication}                                      
& \multicolumn{1}{l}{38.67 $_{\pm{0.66}}$}          
& \multicolumn{1}{l}{38.53 $_{\pm{0.32}}$}          
& \multicolumn{1}{l}{37.68 $_{\pm{0.41}}$}         
& \multicolumn{1}{l}{23.89 $_{\pm{0.84}}$}         
& \multicolumn{1}{l}{23.95 $_{\pm{0.72}}$}          
&  \multicolumn{1}{l}{23.48 $_{\pm{0.61}}$}                            
\\ 
%\rowcolor{lightgray}
FedAvg+ASD (\textbf{Ours}) 
& \multicolumn{1}{l}{\textbf{42.77} $_{\pm{0.22}}$}    
& \multicolumn{1}{l}{\textbf{42.54} $_{\pm{0.51}}$}   
& \multicolumn{1}{l}{\textbf{43.00} $_{\pm{0.60}}$}    
& \multicolumn{1}{l}{\textbf{25.31} $_{\pm{0.25}}$}    
& \multicolumn{1}{l}{\textbf{26.38} $_{\pm{0.21}}$}    
& \multicolumn{1}{l}{\textbf{26.67} $_{\pm{0.10}}$}           
\\ 
\midrule
FedProx~\citep{li2020federated}                                     
& \multicolumn{1}{l}{37.79 $_{\pm{0.97}}$}          
& \multicolumn{1}{l}{37.92 $_{\pm{0.55}}$}          
& \multicolumn{1}{l}{ 37.94$_{\pm{0.22}}$}         
& \multicolumn{1}{l}{24.61 $_{\pm{1.24}}$}         
& \multicolumn{1}{l}{23.57 $_{\pm{0.44}}$}          
&  \multicolumn{1}{l}{23.27 $_{\pm{0.11}}$}                            
\\
%\rowcolor{lightgray}
FedProx+ASD (\textbf{Ours}) 
& \multicolumn{1}{l}{\textbf{41.31} $_{\pm{0.90}}$}    
& \multicolumn{1}{l}{\textbf{41.67} $_{\pm{0.12}}$}   
& \multicolumn{1}{l}{\textbf{42.30} $_{\pm{0.37}}$}    
& \multicolumn{1}{l}{\textbf{25.49} $_{\pm{0.45}}$}    
& \multicolumn{1}{l}{\textbf{25.62} $_{\pm{0.05}}$}    
& \multicolumn{1}{l}{\textbf{25.58} $_{\pm{0.18}}$} 
\\ 
\midrule
FedNTD~\citep{lee2021preservation}                                       
& \multicolumn{1}{l}{40.40 $_{\pm{1.52}}$}          
& \multicolumn{1}{l}{40.50 $_{\pm{0.54}}$}          
& \multicolumn{1}{l}{41.23 $_{\pm{0.44}}$}            
& \multicolumn{1}{l}{23.71 $_{\pm{0.65}}$}          
& \multicolumn{1}{l}{23.28 $_{\pm{0.29}}$}           
& \multicolumn{1}{l} {22.95 $_{\pm{0.22}}$}                                     
\\ 
%\rowcolor{lightgray}
FedNTD+ASD (\textbf{Ours})                       
& \multicolumn{1}{l}{\textbf{43.01} $_{\pm{0.34}}$}          
& \multicolumn{1}{l}{\textbf{43.61} $_{\pm{0.33}}$}          
& \multicolumn{1}{l}{\textbf{43.25} $_{\pm{0.41}}$}     
& \multicolumn{1}{l}{\textbf{27.34} $_{\pm{0.73}}$}          
& \multicolumn{1}{l}{\textbf{27.39} $_{\pm{0.39}}$}          
& \multicolumn{1}{l}{\textbf{27.41} $_{\pm{0.11}}$}                       
\\ 
\midrule
FedDyn~\citep{acar2021federated}                                       
& \multicolumn{1}{l}{47.56 $_{\pm{0.41}}$}          
& \multicolumn{1}{l}{48.60 $_{\pm{0.09}}$}          
& \multicolumn{1}{l} {48.87 $_{\pm{0.51}}$}  
& \multicolumn{1}{l}{27.62 $_{\pm{0.21}}$}          
& \multicolumn{1}{l}{28.58 $_{\pm{0.61}}$}          
&  \multicolumn{1}{l} {28.37 $_{\pm{0.20}}$}                               
\\
%\rowcolor{lightgray}
FedDyn+ASD (\textbf{Ours})                       
& \multicolumn{1}{l}{\textbf{49.03}  $_{\pm{0.24}}$} 
& \multicolumn{1}{l}{\textbf{50.23}  $_{\pm{0.25}}$} 
&  \multicolumn{1}{l}{\textbf{51.44} $_{\pm{0.48}}$} 
& \multicolumn{1}{l}{\textbf{29.94}  $_{\pm{0.67}}$} 
& \multicolumn{1}{l}{\textbf{30.05}  $_{\pm{0.24}}$} 
& \multicolumn{1}{l}{\textbf{30.76}  $_{\pm{0.44}}$}
\\ 
\midrule
FedSAM~\citep{caldarola2022improving}                                      
& \multicolumn{1}{l}{40.89 $_{\pm{0.30}}$}          
& \multicolumn{1}{l}{41.41 $_{\pm{0.34}}$}          
& \multicolumn{1}{l} {40.81 $_{\pm{0.26}}$}  
& \multicolumn{1}{l}{24.72 $_{\pm{0.64}}$}          
& \multicolumn{1}{l}{25.42 $_{\pm{0.49}}$}          
&  \multicolumn{1}{l} {23.50 $_{\pm{0.94}}$}                               
\\
%\rowcolor{lightgray}
FedSAM+ASD (\textbf{Ours})                       
& \multicolumn{1}{l}{\textbf{43.99}  $_{\pm{0.14}}$} 
& \multicolumn{1}{l}{\textbf{44.54}  $_{\pm{0.30}}$} 
&  \multicolumn{1}{l}{\textbf{44.77} $_{\pm{0.11}}$} 
& \multicolumn{1}{l}{\textbf{26.26}  $_{\pm{0.47}}$} 
& \multicolumn{1}{l}{\textbf{26.80}  $_{\pm{0.17}}$} 
& \multicolumn{1}{l}{\textbf{25.37}  $_{\pm{0.26}}$}
\\ 
\midrule
FedDisco~\citep{ye2023feddisco}                                       
& \multicolumn{1}{l}{38.97 $_{\pm{1.38}}$}          
& \multicolumn{1}{l}{38.87 $_{\pm{1.37}}$}          
& \multicolumn{1}{l} {37.85 $_{\pm{0.57}}$}  
& \multicolumn{1}{l}{24.35 $_{\pm{0.42}}$}          
& \multicolumn{1}{l}{24.03 $_{\pm{0.78}}$}          
&  \multicolumn{1}{l} {23.49 $_{\pm{0.31}}$}                               
\\
%\rowcolor{lightgray}
FedDisco+ASD (\textbf{Ours})                       
& \multicolumn{1}{l}{\textbf{41.55}  $_{\pm{1.06}}$} 
& \multicolumn{1}{l}{\textbf{41.94}  $_{\pm{0.30}}$} 
&  \multicolumn{1}{l}{\textbf{43.09} $_{\pm{0.47}}$} 
& \multicolumn{1}{l}{\textbf{25.43}  $_{\pm{0.46}}$} 
& \multicolumn{1}{l}{\textbf{26.03}  $_{\pm{0.26}}$} 
& \multicolumn{1}{l}{\textbf{26.56}  $_{\pm{0.9}}$}
\\ 
\midrule
FedSpeed~\citep{sun2023fedspeed}                                       
& \multicolumn{1}{l}{47.39 $_{\pm{0.82}}$}          
& \multicolumn{1}{l}{48.27 $_{\pm{0.13}}$}          
& \multicolumn{1}{l} {49.01 $_{\pm{0.46}}$}  
& \multicolumn{1}{l}{28.60 $_{\pm{0.15}}$}          
& \multicolumn{1}{l}{29.33 $_{\pm{0.3}}$}          
&  \multicolumn{1}{l} {29.62 $_{\pm{0.31}}$}                               
\\
%\rowcolor{lightgray}
FedSpeed+ASD (\textbf{Ours})                       
& \multicolumn{1}{l}{\textbf{49.16}  $_{\pm{0.40}}$} 
& \multicolumn{1}{l}{\textbf{49.76}  $_{\pm{0.27}}$} 
&  \multicolumn{1}{l}{\textbf{51.99} $_{\pm{0.32}}$} 
& \multicolumn{1}{l}{\textbf{30.97}  $_{\pm{0.25}}$} 
& \multicolumn{1}{l}{\textbf{30.05}  $_{\pm{0.24}}$} 
& \multicolumn{1}{l}{\textbf{32.68}  $_{\pm{0.53}}$}
\\ 
\bottomrule
\end{tabular}
}
\label{tab:table_acc}
\end{table*}

\begin{figure*}[htp]
  \centering
  \subfigure[CIFAR-100 ($\delta = 0.3$)]{\includegraphics[scale=0.34]{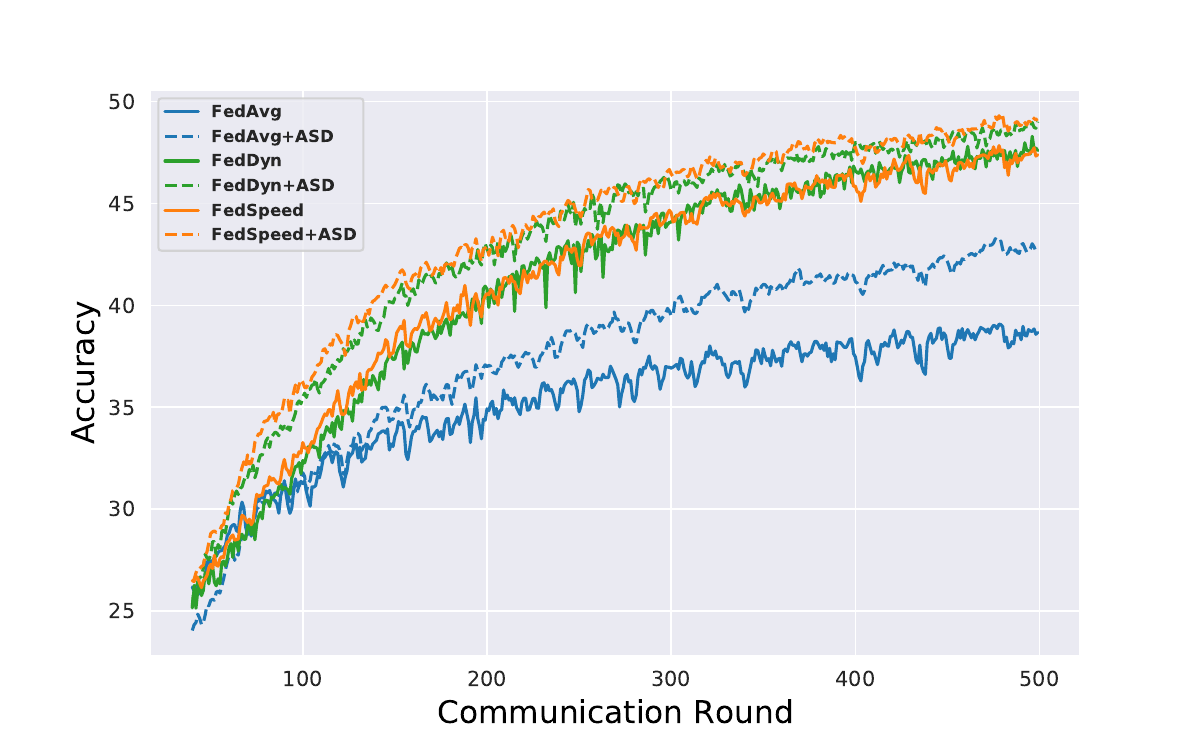}}\hspace{-2.2em}
  \subfigure[Tiny-ImageNet ($\delta = 0.3$)]{\includegraphics[scale=0.34]{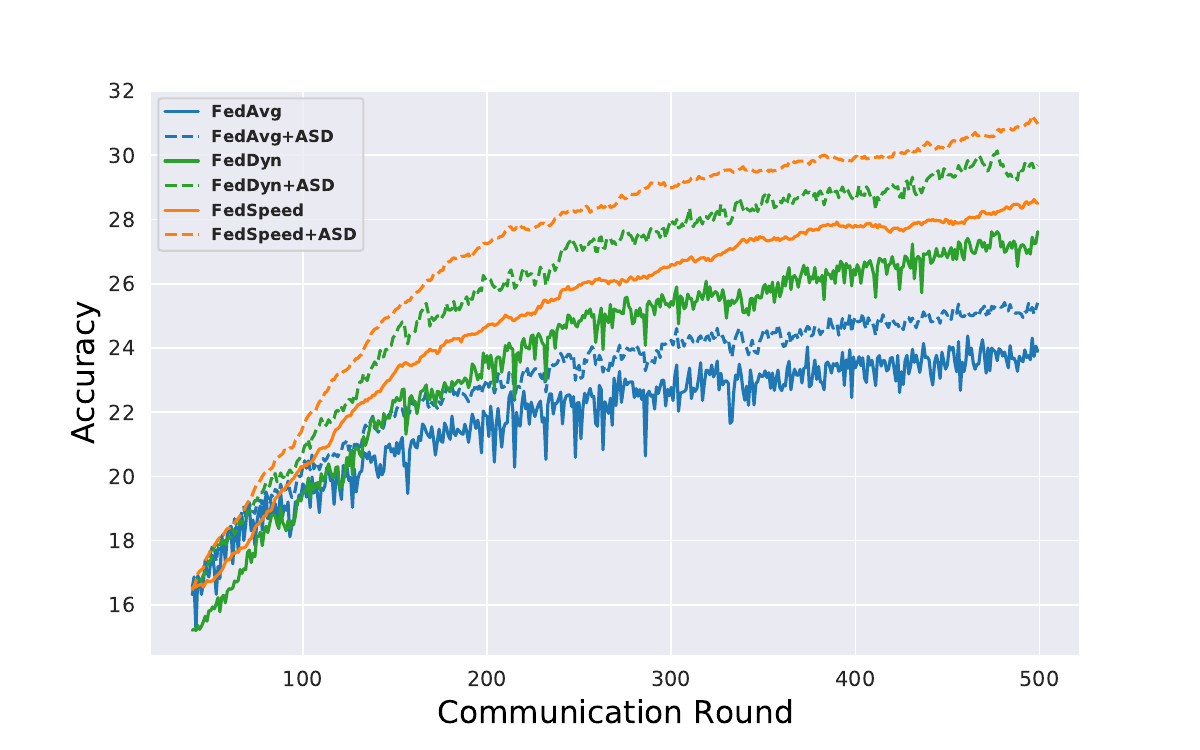}}\hspace{-2.2em}
  %%\vspace{-0.12in}
  \caption{Test Accuracy vs Communication rounds: Comparison of algorithms with $\delta = 0.3$ partitions on CIFAR-100 and Tiny-ImageNet datasets. All the algorithms augmented with proposed regularization (ASD) outperform compared to their original form. FedSpeed+ASD outperforms all the other algorithms.}
  \label{fig_delta_0pt3}
\end{figure*}

\begin{figure*}[htp]
  \centering
 \subfigure[CIFAR-100 ($\delta = 0.6$)]{\includegraphics[scale=0.34]{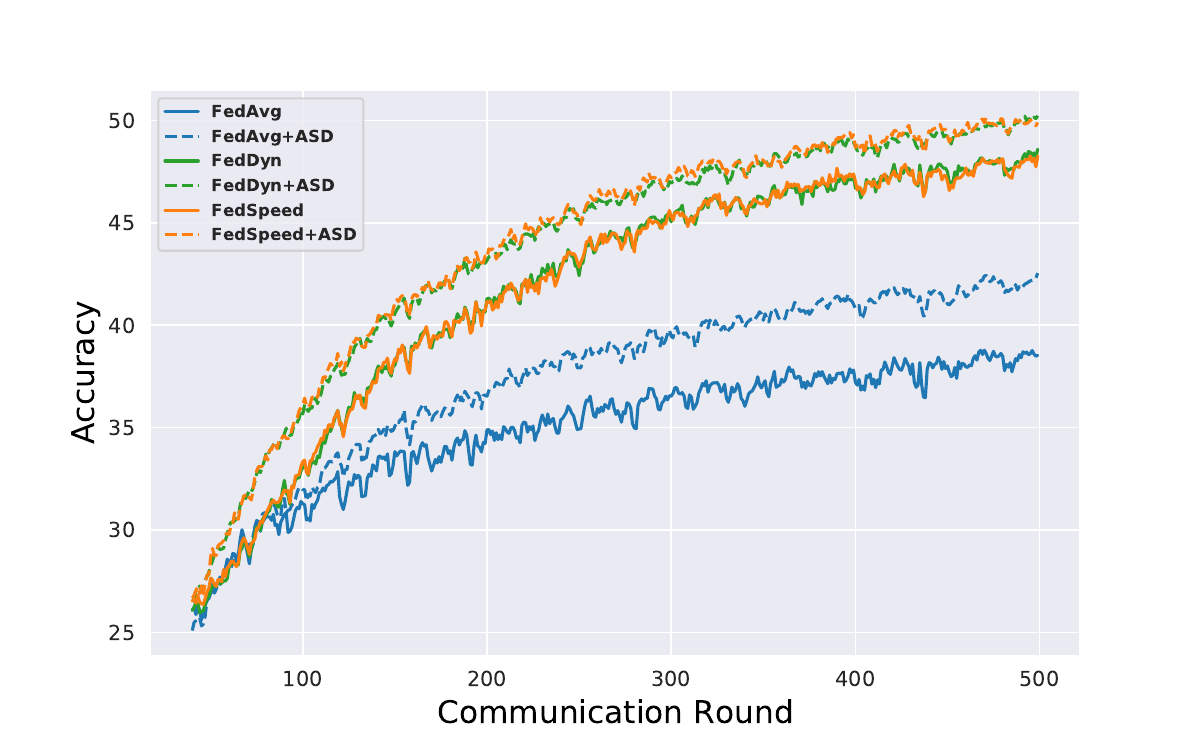}}\hspace{-2.2em}
  \subfigure[Tiny-ImageNet ($\delta = 0.6$)]{\includegraphics[scale=0.34]{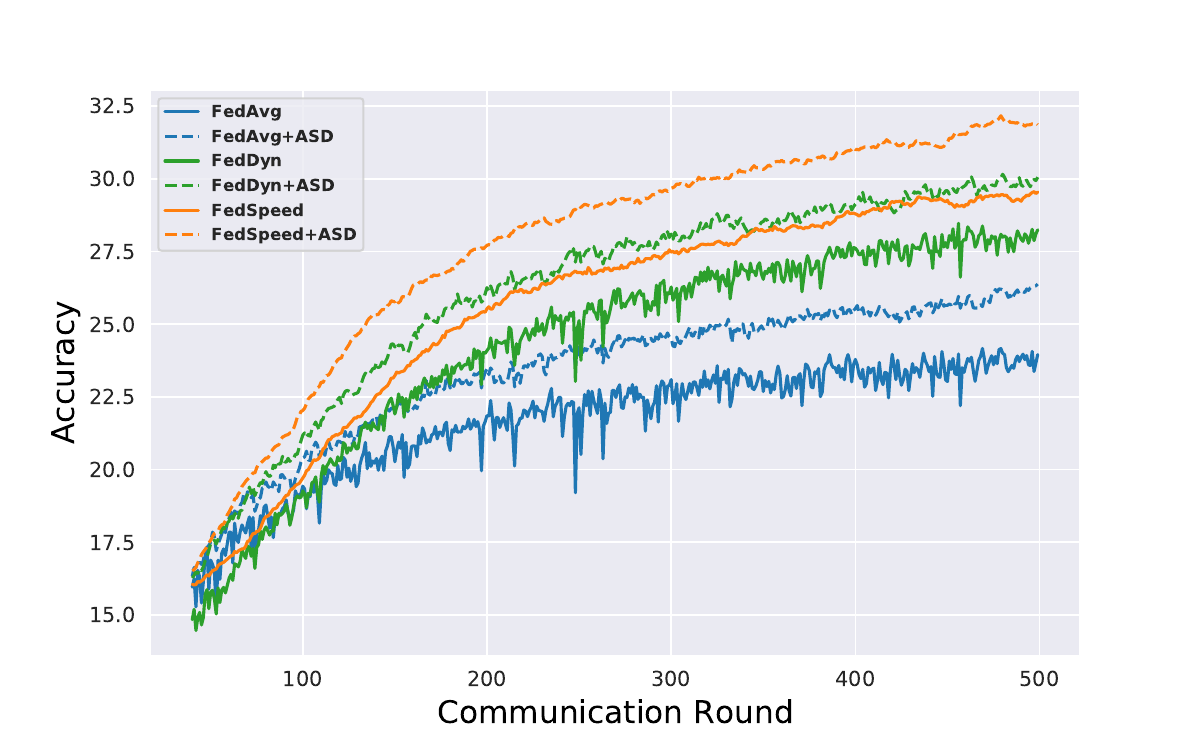}}\hspace{-2.2em}
  %\vspace{-0.12in}
  \caption{Test Accuracy vs Communication rounds: Comparison of algorithms with  $\delta = 0.6$ data partitions on CIFAR-100 and Tiny-ImageNet dataset. All the algorithms augmented with proposed regularization (ASD) outperform compared to their original form. FedSpeed+ASD outperforms all the other algorithms.% except iid case where it performs close to FedDyn
  }
  \label{fig_delta_0pt6}
\end{figure*}

\begin{figure*}[htp]
  \centering
    \subfigure[CIFAR-100 (iid)]{\includegraphics[scale=0.34]{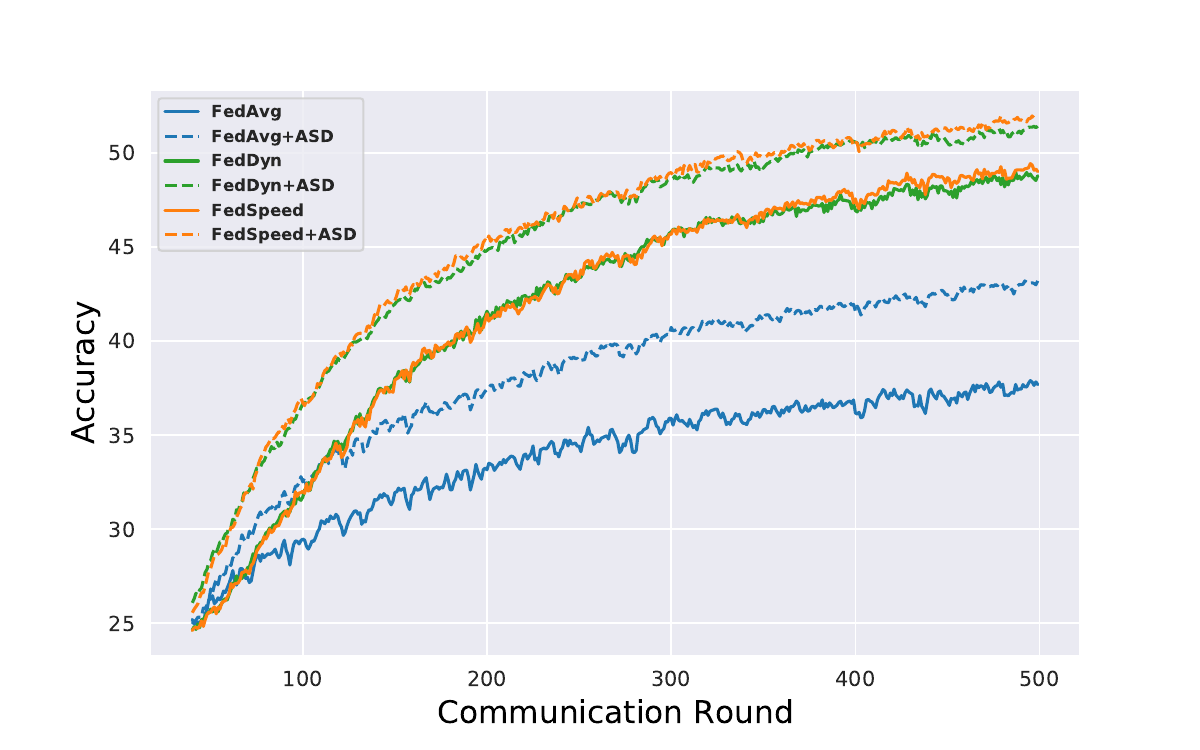}}\hspace{-2.2em}
  \subfigure[Tiny-ImageNet (iid)]{\includegraphics[scale=0.34]{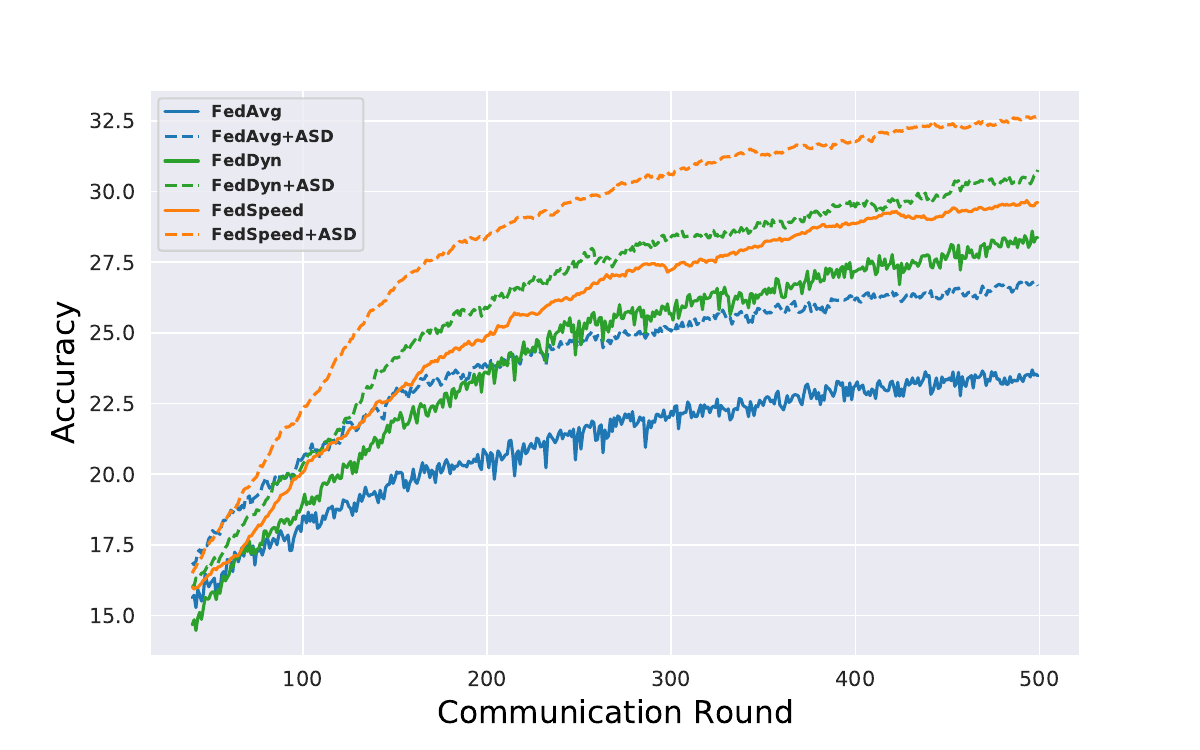}}\hspace{-2.2em}
  %%\vspace{-0.12in}
  \caption{Test Accuracy vs Communication rounds: Comparison of algorithms with iid data partitions on CIFAR-100 and Tiny-ImageNet datasets. All the algorithms augmented with proposed regularization (ASD) outperform compared to their original form. FedSpeed+ASD outperforms all the other algorithms.}
  \label{fig_delta_iid}
\end{figure*}
%\vspace{-0.2in}

\section{Results and Discussion}
%\vspace{-0.1in}
For evaluation, we report accuracy on the test dataset as our performance metric and the number of communication rounds required to attain the desired accuracy as a metric to quantify the communication cost. Specifically, we evaluate the global model on the test set and report its accuracy after every communication round. For comparison, we consider the popular methods for federated learning, such as FedAvg, FedProx, FedDyn, FedSpeed FedNTD,
FedSAM and 
FedDisco. We augment each of these methods with our approach (ASD)
%the FedAvg and FedDyn 
and observe a significant boost in performance. For a fair comparison, we consider the same models used in Fedavg \citep{mcmahan2017communication}, and FedDyn \citep{acar2021federated}, for CIFAR-10 and CIFAR-100 classification tasks. The model architecture used for CIFAR-100 contains $2$ convolution layers followed by $3$ fully connected layers. For Tiny-ImageNet, we use $3$ convolution followed by $3$ fully connected layers. The detailed architectures are given in Sec~\ref{model_arch} of the appendix. 
Hyperparameters: SGD algorithm with a learning rate of 0.1 and decay the learning rate per round of 0.998 is used to train the client models. Temperature $\tau$ is set to 2.0. We only tune the hyper-parameter $\lambda$. More hyperparameter setting details and impact of  $\lambda$, $\tau$ are provided in Sec.~\ref{sup:hyp_settings} and~\ref{sup:hyper_sensitivity} of the appendix, respectively. %\textcolor{mycolor}
{The impact of client participation rate and the number of clients on ASD are shown in Sec.~\ref{sup:cp_num_clnts} of the appendix.} Implementation of ASD with other FL methods is discussed in Sec.~\ref{asd_plus_algs} of appendix.
%The $\beta$, and $\beta^{kl}$ are always set to $1$ and 
%in principle, one can fine-tune these values to obtain the best accuracy. To keep the method simple, we don't fine-tune these values andthe value of $\lambda$ is set to $20$ and $30$ for CIFAR-100 and Tiny-ImageNet respectively for all our experiments. 
We compare the convergence of different schemes for 500 communication rounds. Following the testing protocol of \citep{acar2021federated}, we average across all the client models and compute the test accuracy on the averaged model, which is reported in our results. In all the tables, we report the test accuracy of the global model in $\%$ at the end of 500 communication rounds. All the experiments in the tables are performed over three different initializations, mean and standard deviations of accuracy over the three experiments are reported. We also demonstrate the efficacy of our proposed method with deeper architectures such as ResNet-20 and  %\textcolor{mycolor}
{Vision Transformer (ViT)} models in Sec~\ref{sup:deep_models} and %\textcolor{mycolor}
{ Sec~\ref{sup:deep_models_vit}} of the appendix respectively.
%In the figures~\ref{fig_cifar100} and~\ref{fig_tiny} we report the test accuracy after every round.
\par
%We also compare with the related methods such as FedIR~\citep{hsu2020federated} and FedProx~\citep{li2020federated} in the section 3 of the appendix. The ablation study of KL term and the information term is also given in section 7 of the appendix.  

\vspace{-0.05in}
\subsection{Performance of ASD on CIFAR-10/100 and Tiny-Imagenet}
\vspace{-0.05in}
In Table~\ref{tab:table_acc}, we report the performance of CIFAR-100 and Tiny-ImageNet datasets with various algorithms for non-iid ($Dir(\delta = 0.3)$ and $Dir(\delta = 0.6)$) as well as the iid settings. Each experiment is performed over three different initializations, and the mean and standard deviation of the accuracy are reported.
On CIFAR-100, we observe that our proposed ASD applied on FedDyn improves its performance by $\approx 1.45\%$ for $Dir(\delta = 0.3)$ and $\approx 1.6\%$ $Dir(\delta = 0.6)$.
 Similarly, for Tiny-ImageNet we observe that FedDyn+ASD improves FedDyn by $\approx 2.4\%$ for $Dir(\delta = 0.3)$ and by $\approx 1.4\%$ for $Dir(\delta = 0.6)$. The test accuracy vs communication rounds plot on CIFAR-100 and Tiny-ImageNet datasets is shown in Figures~\ref{fig_delta_0pt3},~\ref{fig_delta_0pt6}~\ref{fig_delta_iid} across non-iid and iid partitions.
 %The accuracy vs communication rounds plots is shown in Figure~\ref{fig_cifar100}.
We can see that adding ASD gives consistent improvement across the rounds\footnote{In the figures we only compare FedAvg, FedDyn and FedSpeed for better readability. For others, please refer to Sec~\ref{sup:acc_comm} of the appendix}. 
%A similar plot for the Tiny-Imagenet dataset is shown in Figure~\ref{fig_tiny}.
%The improvement in the iid case for Tiny-ImageNet dataset is marginal because we did not tune the hyper-parameter $\lambda$ for every experiment, intuitively the lower value of $\lambda$ will be better for iid case.
We obtain significant improvements for FedAvg+ASD against FedAvg, FedProx+ASD against FedProx, FedSpeed+ASD against FedSpeed, etc. In Sec~\ref{sup:cifar10} of the appendix we present the CIFAR-10 results where we observe that adding ASD consistently gives an improvement of $\approx 0.4\% - 1.99\% $ improvement across the algorithms. %Further ablations are provided in Sec.~\ref{sup:abl_study} of the appendix.
%To obtain smoother predictions, we average across all the client models and compute the test accuracy. We report the same in the plots and tables as done in \citep{acar2021federated}.

%\vspace{-0.05in}
\subsection{Comparison with adaptive vs uniform weights}
%\vspace{-0.05in}
\begin{table}[htp]
\centering 
\caption{Comparison with adaptive weights vs uniform weights on CIFAR-100 dataset with Dirichlet $\delta = 0.3$}
\scalebox{0.8}{
\begin{tabular}{l|l|l}
\hline
\multirow{2}{*}{Algorithm} 
& \multicolumn{1}{c|}{\begin{tabular}[c]{@{}c@{}} Distillation with\\ Uniform weights\end{tabular}} 
& \multicolumn{1}{c}{\begin{tabular}[c]{@{}c@{}} Distillation with\\ Adaptive  weights\end{tabular}} \\ 
%\cline{2-3}                                 
&   &                                                     
\\ \hline
FedAvg+ASD                            
& 41.75  $_{\pm{0.12}}$                                                                        
& \textbf{42.77  $_{\pm{0.22}}$ }                                                                   
\\ \hline
FedNTD+ASD                             
& 40.40 $_{\pm{1.52}}$                                                                               
& \textbf{43.01 $_{\pm{0.34}}$ } 

\\ \hline

FedDisco+ASD                             
& 40.21 $_{\pm{0.57}}$                                                                               
& \textbf{41.55 $_{\pm{1.06}}$ } 

\\ \hline

FedDyn+ASD                             
& 47.90 $_{\pm{0.35}}$                                                                               
& \textbf{49.03 $_{\pm{0.24}}$ }   

%\\ \hline

%FedSpeed+ASD                             
%& 48.03 $_{\pm{0.46}}$                                                                         
%& \textbf{49.16 $_{\pm{0.40}}$ }  
\\ \hline
\end{tabular}
}
\label{sup:tab:table6}
\end{table}
 We analyze the impact of the proposed adaptive weighting scheme.
We compare by making all the $\hat{\alpha_{k}}^{i}$ in Eq~\ref{alpha_hat_eq} to $1$ i.e, by giving equal weights to all the samples in the mini-batch. We can see from Table~\ref{sup:tab:table6} that the proposed adaptive weighting scheme yields much better performance than assigning uniform weights, thus establishing the impact of proposed adaptive weights. %In iid cases, it impacts very marginally.

%\textcolor{mycolor}
\subsection{Performance with increased clients and lower client participation}
\begin{table}[htp]
\centering
%\textcolor{mycolor}
{
\centering
\caption{Experiments with 500 clients}
\scalebox{0.8}{
\begin{tabular}{c|c}
\toprule
Method       & Accuracy (in \%) \\ \hline
FedAvg       & 27.92            \\ 
FedAvg+ASD (ours)   & \textbf{31.28}   \\ \hline
FedProx      & 28.09            \\ 
FedProx+ASD  & \textbf{32.13}   \\ \hline
FedNTD      & 30.99           \\ 
FedNTD+ASD  & \textbf{33.65}   \\ \hline
FedDyn       & 31.0            \\ 
FedDyn+ASD (ours)  & \textbf{33.12}   \\ \hline
FedSpeed     & 34.08            \\
FedSpeed+ASD (ours) & \textbf{36.59}   \\ \bottomrule
\end{tabular}
}
\label{exp_increase_cl}
}
\end{table}

%\textcolor{mycolor}
{
In this section, we analyze the impact of our ASD regularizer to mimic the cross-device setting. We increase the client participation to 500 clients and only 1\% of the clients participate in every round. We consider the CIFAR-100 dataset and the non-iid data partition of $\delta = 0.3$. In the Table~\ref{exp_increase_cl}, we observe that ASD consistently improves the performance of the algorithms. The accuracies are reported after averaging over three different initializations at the end of $1000$ communication rounds. Even in this challenging setting ASD consistently improves the performance of the FL algorithms.} 

\section{Computation Cost} 
\label{sec_compute}
\vspace{-0.1in}
 The major computation for the distillation scheme comes from the teacher forward pass, student forward pass, and the student backward pass~\citep{xu2020computation}. We assume $C_s$ as the total computational cost of server model forward pass and $C_k$ be the total computation cost of client model $k$ the forward pass per epoch. We do not need $C_s$ computations every epoch, we only need to compute once and store the values of $\mathcal{H}(x)$ while keeping the same backward computation. Specifically, for the computation of distillation regularizer we only need $E*C_k + C_s$ local computations compared to $E*C_k$ computations without regularizer. Here $E$ denotes the local epochs of the client. Since $C_s = C_k$, we have $(E+1)*C_k$ local computations. Thus, our regularizer introduces minimal forward computation on the edge devices, which typically have low computation. 
 %\textcolor{mycolor}
 {In  Sec.~\ref{comp_vs_acc} of appendix we discuss the computation vs accuracy of ASD.}        
\iffalse
Let $F_{t}$, $F_{s}$ and $B_{s}$ denote the number of operations for teacher forward pass, student forward pass and student backward pass respectively. Overall cost with batch of $N_{b}$ samples given below. 
%%\vspace{-0.1in}
\begin{equation} 
C_{reg} = (F_{t} + F_{s} + B_{s})*N_{b}
\end{equation}
%%\vspace{-0.17in}
\begin{equation}
C_{noreg} = (F_{s} + B_{s})*N_{b}
\end{equation}
$C_{reg}$ is the cost due to the proposed regularization cost and $C_{noreg}$ is the cost without regularization. It can be clearly seen that we increase the computation cost per batch due to the teacher model by $N_b * F_{t}$. The cost is linear in batch size and since we use the same model architecture for student and teacher we have $F_t = F_s$ this implies we double the forward pass computation.
If storing the model is too expensive then one can get all the predictions, i.e., the softmax probabilities and store them instead of the model. This way one can save the memory and the repeated computations of the global model. 
\fi 
%%\vspace{-0.1in}
\section{Conclusion}
\vspace{-0.1in}
In this work, we presented an efficient and effective method for addressing client data heterogeneity due to label imbalance in federated learning using our proposed  Adaptive Self-Distillation (ASD), which does not require any auxiliary data and no extra communication cost. We also theoretically showed that ASD has lower client-drift leading to better convergence. Moreover, we performed analysis to show that ASD has better generalization by analyzing the top eigenvalue and trace  of the Hessian of the global model's loss. The effectiveness of our approach is shown via extensive experiments across datasets such as CIFAR-10, CIFAR-100 and Tiny-ImageNet with different degrees of heterogeneity. Our proposed regularizer (ASD) can be integrated easily atop any of the FL frameworks. We evaluated this efficacy by showing improvement in the performance when combined with FedAvg, FedProx, FedDyn, FedSAM, FedDisco, FedNTD and FedSpeed. We have also shown that the computation required to implement ASD is simply an additional forward pass on the client-side training, i.e, all the gains we obtain with ASD requires minimal compute. Our research can inspire the designing of the computationally efficient regularizers that concurrently reduce client-drift and improve the generalization.

\bibliography{tmlr}
\bibliographystyle{tmlr}

\appendix
\section{Appendix}
%You may include other additional sections here.

\subsection{Notations and Definitions}
\begin{itemize}
   
  \item$\mathcal{H}(x)$ denotes the entropy of the model under consideration for input $x$.
  \item$p_k^{y^i = c}$ denotes the probability that the client $k$ has the input $i$ belonging to class $c$.
  \item$\mathbf{H}(g)$ denotes the Hessian of the function $g$.
  \item$\mu_1(\mathbf{A})$ denotes the top eigenvalue of matrix $\mathbf{A}$.
  \item$\mathcal{D}_{KL}$ denotes the KL divergence.
  \item$\inf$ denotes the infimum and $\sup$ denotes the supremum.
  \item$\delta$ is used for denoting the heterogeneity generated based on Dirichlet distribution.
  \item$\lambda$ denotes the ASD regularizer strength.
  \item$G_d(\mathbf{w},\lambda)$ denotes the gradient dissimilarity.
  \item$\mathbb{E}(.)$ denotes the expectation. 
  \item$L_{k}(\mathbf{w})$ is the loss of client $k$ (cross-entropy loss). 
  \item$ {L_{k}(\mathbf{w})}^{ASD}$ is the Adaptive Self-Distillation loss for client $k$.
  \item$\mathcal{D}_k$ represents the dataset of client $k$.
  \item $\mathcal{P}_k(x,y)$ represents the data distribution of the client $k$.
\end{itemize}

\subsection{Model Architectures}
\label{model_arch}
In Table~\ref{model_table}, the model architecture is shown. We use PyTorch style representation. For example conv layer($3$,$64$,$5$) means $3$ input channels, $64$ output channels and the kernel size is $5$. Maxpool($2$,$2$) represents the kernel size of $2$ and a stride of $2$. FullyConnected(384,200) represents an input dimension of $384$ and an output dimension of $200$. The architecture for CIFAR-100 is exactly the same as used in~\citep{acar2021federated}. 
\begin{table}[htp]
\centering
\caption{Models used for Tiny-ImageNet and CIFAR-100 datasets.}
\scalebox{0.7}{
\begin{tabular}{c|c}
\hline
\multicolumn{1}{l|}{\multirow{6}{*}{\textbf{CIFAR-10/100 Model}}} & \textbf{Tiny-ImageNet Model} \\ \cline{2-2} 
\multicolumn{1}{l|}{}                                              & ConvLayer(3,64,3)            \\ \cline{2-2} 
\multicolumn{1}{l|}{}                                              & GroupNorm(4,64)               \\ \cline{2-2} 
\multicolumn{1}{l|}{}                                              & Relu                          \\ \cline{2-2} 
\multicolumn{1}{l|}{}                                              & MaxPool(2,2)                  \\ \cline{2-2} 
\multicolumn{1}{l|}{}                                              & ConvLayer(64,64,3)           \\ \hline
\multicolumn{1}{l|}{}                                              & GroupNorm(4,64)               \\ \hline
ConvLayer(3,64,5)                                                  & Relu                          \\ \hline
Relu                                                                & MaxPool(2,2)                  \\ \hline
MaxPool(2,2)                                                        & ConvLayer(64,64,3)           \\ \hline
ConvLayer(64,64,5)                                                 & GroupNorm(4,64)               \\ \hline
Relu                                                                & Relu                          \\ \hline
MaxPool(2,2)                                                        & MaxPool(2,2)                  \\ \hline
Flatten                                                             & Flatten                       \\ \hline
FullyConnected(1600,384)                                           & FullyConnected(4096,512)     \\ \hline
Relu                                                                & Relu                          \\ \hline
FullyConnected(384,192)                                            & FullyConnected(512,384)      \\ \hline
Relu                                                                & Relu                          \\ \hline
FullyConnected(192,100)                                                     & FullyConnected(384,200)               \\ \hline
\end{tabular}
}

\label{model_table}
\end{table}

\subsection{Hyper-Parameter Settings}
\label{sup:hyp_settings}
 The value of $\lambda$ is specified in units of batch-size $B$. We chose $\lambda$ from $\{10,20,30\}$. We set $\lambda =20$ for all the Tiny-ImageNet experiments. For CIFAR-10/100 we chose $\lambda$ to be $10$ and $30$ respectively. The  batch-size ($B$) of 50 and learning rate of 0.1 with decay of 0.998 is employed for all the experiments unless specified. All the experiments are carried out with 100 clients and with 10\% client participation.    

\subsection{Experiments with Deeper Models (CNN's)}
\label{sup:deep_models}
In this section, we perform experiments with the deep models such as ResNet-20 on CIFAR-100 dataset with Dirichlet $\delta = 0.3$. For this experiment we have used 300 communication rounds, the number of clients as 30, and the  client participation rate is set to 20\%. In the table~\ref{res_20exp}, we report the numbers averaged over 3 different trials. We observe that the addition of our proposed regularizer ASD atop mutiple popular FL methods leads to consistent improvements, thereby further justifying the efficacy of our proposed method.
 
\begin{table}[htp]
\centering
\caption{Experiments on ResNet-20}
\scalebox{0.8}{
\begin{tabular}{c|c}
\toprule
Method       & Accuracy (in \%) \\ \hline
FedAvg       & 46.35            \\ 
FedAvg+ASD (ours)   & \textbf{47.90}   \\ \hline
%FedProx      & 48.06            \\ \hline
%FedProx+ASD  & \textbf{49.03}   \\ \hline
FedDyn       & 53.60            \\ 
FedDyn+ASD (ours)  & \textbf{55.15}   \\ \hline
FedSpeed     & 54.42            \\
FedSpeed+ASD (ours) & \textbf{55.82}   \\ \bottomrule
\end{tabular}
}
\label{res_20exp}
\end{table}

%\textcolor{mycolor}
{
\subsection{Experiments with Deeper Models (ViT)}
\label{sup:deep_models_vit}}
%\textcolor{mycolor}
{
We perform experiments with ViT architecture using the Tiny-ViT~\citep{wu2022tinyvit} as client models  on ImageNet-100 dataset~\citep{zang2022dlme} with non-iid data partitioning of Dirichlet $\delta = 0.3$. The choice of  Tiny-ViT is motivated by the fact that edge devices are traditionally computational resource-constrained and Tiny-ViT is designed for such applications. For this experiment, the number of clients is set to 200, and the client participation rate is set to 5\%. We have used 300 communication rounds.  In the Table~\ref{sup:tiny_vit_imnet}, we report the numbers averaged over 3 different trials. We observe that the addition of our proposed regularizer ASD atop FedAvg and FedDyn leads to consistent improvements, thereby further justifying the efficacy of our proposed method on the deeper architectures. 
%These experiments are added  to the Sec.A.5 of the appendix in the paper.
%In this section, we perform experiments with ViT architectures such as the Tiny ViT ImageNet-100 dataset with noon-iid data partitioning of Dirichlet $\delta = 0.3$. For this experiment, we have used 300 communication rounds, the number of clients as 200, and the  client participation rate is set to 5\%. In the Table~\ref{sup:tiny_vit_imnet}, we report the numbers averaged over 3 different trials. We observe that the addition of our proposed regularizer ASD atop mutiple popular FL methods leads to consistent improvements, thereby further justifying the efficacy of our proposed method.
}

\begin{table}[htp]
\centering
%\textcolor{mycolor}
{
\centering
\caption{Experiments using Tiny-ViT on ImageNet-100 dataset with the non-iid partitioning of $\delta=0.3$}
\scalebox{0.8}{
\begin{tabular}{c|c}
\toprule
Method       & Accuracy (in \%) \\ \hline
FedAvg       & 18.12            \\ 
FedAvg+ASD (ours)   & \textbf{22.10}   \\ \hline
%FedProx      & 48.06            \\ \hline
%FedProx+ASD  & \textbf{49.03}   \\ \hline
FedDyn       & 28.02            \\ 
FedDyn+ASD (ours)  & \textbf{36.70}   \\ \bottomrule
\end{tabular}
}
\label{sup:tiny_vit_imnet}
}
\end{table}
%\textcolor{mycolor}
{
In Table~\ref{sup:small_vit_cifar}, we performed an experiment on ViT-Small architecture on CIFAR-100. We observe that adding our ASD regularizer improves the baseline FedAvg by $1.4$\% and $1.7\%$ for $\delta = 0.3$ and $\delta = 0.6$, respectively. In this setup we consider $100$ clients with 10\% participation and the accuracy is reported at the end of $300$ rounds.}

\begin{table}[htp]
\centering
%\textcolor{mycolor}
{
\centering
\caption{Experiments using ViT-Small with CIFAR-100 with the non-iid data partitioning of  $\delta=0.3$ and $\delta=0.6$.}
\scalebox{0.8}{
\begin{tabular}{l|lr}
\hline
\multirow{2}{*}{Method} & \multicolumn{2}{c}{Accuracy(\%)}                                         \\ \cline{2-3} 
                        & \multicolumn{1}{l|}{$\delta = 0.3$}   & \multicolumn{1}{l}{$\delta = 0.6$} \\ \hline
FedAvg                  & \multicolumn{1}{r|}{53.22}            & 52.63                               \\ \hline
FedAvg+ASD (Ours)              & \multicolumn{1}{r|}{\textbf{54.67}} & \textbf{54.34}                    \\ \hline
\end{tabular}
}
\label{sup:small_vit_cifar}
}
\end{table}

%\textcolor{mycolor}
\subsection{Impact on the choice of hyperparameters $\lambda$ and $\tau$}
%}
\label{sup:hyper_sensitivity}
%\textcolor{mycolor}
\begin{figure}[htp]
  \centering
  \includegraphics[width=0.4\linewidth]{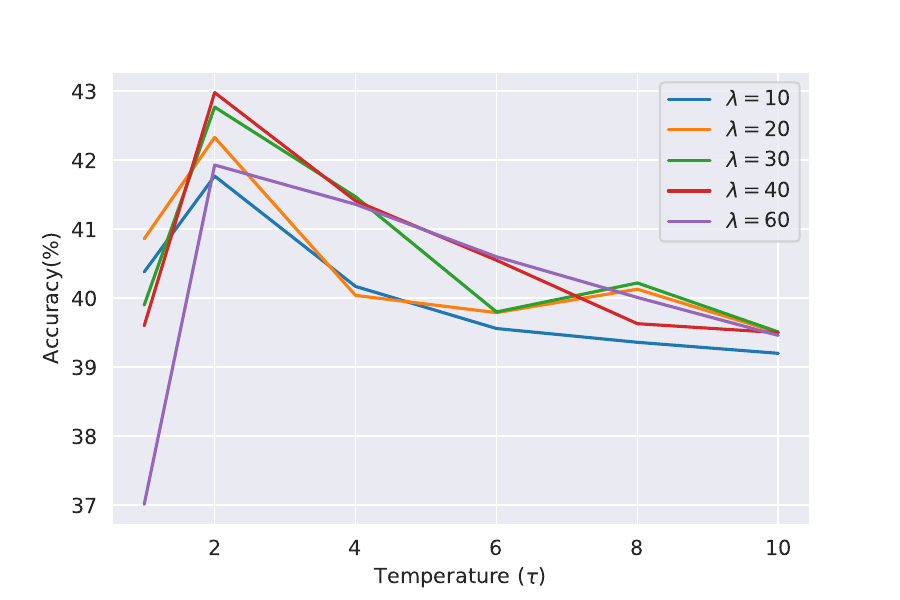}
  \caption{Impact of $\lambda$ and $\tau$ on CIFAR-100 dataset with non-iid partitioning of $\delta = 0.3$ with FedAvg+ASD.}
  \label{fig_hyper}
\end{figure}
We study the impact of changing the hyper-parameters $\lambda$ and $\tau$ on the CIFAR-100 dataset with the Dirichlet non-iid partition of $\delta = 0.3$. We report the accuracy at the end of 500 rounds. When using FedAvg+ASD algorithm. In Figure~\ref{fig_hyper} we see that the accuracy of the model increases with $\lambda$ and then slightly drops after a critical point. This is expected as too less value of $\lambda$ is similar to FedAvg and very high value of $\lambda$ will ignore the local learning. It can also be seen that for all the values of $\lambda$ the Accuracy peaks at $\tau = 2$. In all of our experiments we set the temperature parameter $\tau$ set to 2.0. 
%Overall the variation of accuracy and doesn't drastically impact the accuracy.

%\textcolor{mycolor}

\subsection{Impact of client participation / number of clients on ASD}
\label{sup:cp_num_clnts}

 We fix the client participation to 2\% and vary the number of clients from $100$ to $500$. We perform this ablation using the CIFAR-100 dataset with a non-iid Dirichlet data partitioning of $\delta = 0.3$. We summarize our observations in the Tables~\ref{sup:fix_cp_vary_num_clnt} and~\ref{sup:fix_cp_vary_num_clnt_pt6} below. It can be seen that ASD improves the performance of the baselines FedAvg and FedDyn in all the settings. In particular, we would like to highlight the point here that despite increasing the number of clients, the total number of training data samples across all the clients remains constant (for CIFAR-100). Thus as the number of clients increases, the number of data samples per client decreases. This further aggravates the adverse impact of label heterogeneity across clients, and hence accuracy degrades in general.
However, we are happy to observe and report that, even under such a challenging setup, our proposed adaptive self-distillation-based strategy consistently improves the accuracy when combined on top of the existing baseline algorithms.
\begin{table}[htp]
\centering
%\textcolor{mycolor}
{
\centering
\caption{Impact of increasing the number of clients on the Accuracy (\%) when the participation rate is fixed to 2\% and \textbf{non-iid partitioning of $\delta=0.3$}.}
\scalebox{0.9}{
\begin{tabular}{l|rrrrr}
\hline
                                  & \multicolumn{5}{c}{Number of Clients}                                                                                                                              \\ \cline{2-6} 
\multirow{-2}{*}{Method}          & \multicolumn{1}{c|}{100}                          & \multicolumn{1}{c|}{200}   & \multicolumn{1}{c|}{300}   & \multicolumn{1}{c|}{400}   & \multicolumn{1}{c}{500} \\ \hline
{FedAvg}     & \multicolumn{1}{r|}{{ 31.15}}  & \multicolumn{1}{r|}{31.86} & \multicolumn{1}{r|}{30.05}  & \multicolumn{1}{r|}{28.70} & 26.12                     \\ %\hline
{FedAvg+ASD  (ours)} & \multicolumn{1}{r|}{{ \textbf{37.67}}} & \multicolumn{1}{r|}{\textbf{35.56}} & \multicolumn{1}{r|}{\textbf{32.86}} & \multicolumn{1}{r|}{\textbf{29.98}} & \textbf{27.16}                    \\ \hline
FedDyn                            & \multicolumn{1}{r|}{39.17}                        & \multicolumn{1}{r|}{36.11} & \multicolumn{1}{r|}{34.24}  & \multicolumn{1}{r|}{31.09} & 26.87                     \\ %\hline
FedDyn+ASD (ours)                        & \multicolumn{1}{r|}{\textbf{39.34}}                        & \multicolumn{1}{r|}{\textbf{40.08}} & \multicolumn{1}{r|}{\textbf{36.83}} & \multicolumn{1}{r|}{\textbf{33.61}} & \textbf{28.05}                    \\ \hline
\end{tabular}
}
\label{sup:fix_cp_vary_num_clnt}
}

\end{table}

\begin{table}[htp]
\centering
%\textcolor{mycolor}
{
\centering
\caption{Impact of increasing the number of clients on the Accuracy (\%) when the participation rate is fixed to 2\% and \textbf{non-iid partitioning of $\delta=0.6$}.}
\scalebox{0.9}{
\begin{tabular}{l|rrrrr}
\hline
                                  & \multicolumn{5}{c}{Number of Clients}                                                                                                                              \\ \cline{2-6} 
\multirow{-2}{*}{Method}          & \multicolumn{1}{c|}{100}                          & \multicolumn{1}{c|}{200}   & \multicolumn{1}{c|}{300}   & \multicolumn{1}{c|}{400}   & \multicolumn{1}{c}{500} \\ \hline
{FedAvg}     & \multicolumn{1}{r|}{{ 35.17}}  & \multicolumn{1}{r|}{32.06} & \multicolumn{1}{r|}{30.12}  & \multicolumn{1}{r|}{28.31} & 25.71                     \\ %\hline
{FedAvg+ASD  (ours)} & \multicolumn{1}{r|}{{ \textbf{39.61}}} & \multicolumn{1}{r|}{\textbf{35.49}} & \multicolumn{1}{r|}{\textbf{32.16}} & \multicolumn{1}{r|}{\textbf{29.43}} & \textbf{27.61}                    \\ \hline
FedDyn                            & \multicolumn{1}{r|}{37.96}                        & \multicolumn{1}{r|}{36.56} & \multicolumn{1}{r|}{34.71}  & \multicolumn{1}{r|}{30.37} & 26.45                     \\ %\hline
FedDyn+ASD (ours)                        & \multicolumn{1}{r|}{\textbf{39.40}}                        & \multicolumn{1}{r|}{\textbf{40.49}} & \multicolumn{1}{r|}{\textbf{37.48}} & \multicolumn{1}{r|}{\textbf{33.23}} & \textbf{28.60}                    \\ \hline
\end{tabular}
}
\label{sup:fix_cp_vary_num_clnt_pt6}
}

\end{table}

%\textcolor{mycolor}
In Table~\ref{sup:fix_num_clnt_vary_cp}, unlike the previous ablation, here we fix the number of clients to 100 and vary the client participation rate from 5\%, 10\% and 15\%. We consider the CIFAR-100 dataset with non-iid partitioning of ($\delta=0.3$). As expected, the accuracy of the FL-trained models improve with an increase in the client participation rate. We would also like to highlight here that, by adding our proposed ASD strategy consistently improves the accuracy when combined on top of the existing baseline algorithms such as FedAvg and FedDyn.
%, we fix the number of clients to 100 and increase the client participation from $5\%$ to $15\%$. We consider the CIFAR-100 dataset with non-iid partitioning of Dirichlet ($\delta=0.3$). It can be seen that adding ASD consistently improves the baseline algorithms.

\begin{table}[htp]
\centering
%\textcolor{mycolor}
{
\centering
\caption{Impact of increasing the client participation rate on the Accuracy (\%) with number of clients fixed to $100$.}
\scalebox{0.8}{
\begin{tabular}{l|ccc|ccr}
\hline
\multirow{3}{*}{Method} & \multicolumn{3}{c|}{non iid partition ($\delta = 0.3$)}                                                 & \multicolumn{3}{c}{non-iid partition ($\delta=0.6$)}                                                            \\ \cline{2-7} 
                        & \multicolumn{3}{c|}{client paticipation}                                                   & \multicolumn{3}{c}{client participation}                                                             \\ \cline{2-7} 
                        & \multicolumn{1}{c|}{5\%}            & \multicolumn{1}{c|}{10\%}           & 15\%           & \multicolumn{1}{c|}{5\%}            & \multicolumn{1}{c|}{10\%}           & \multicolumn{1}{c}{15\%} \\ \hline
FedAvg                  & \multicolumn{1}{c|}{38.22}          & \multicolumn{1}{c|}{38.67}          & 38.85          & \multicolumn{1}{c|}{39.04}          & \multicolumn{1}{c|}{38.53}          & \multicolumn{1}{c}{38.00}                        \\ 
FedAvg+ASD (Ours)             & \multicolumn{1}{c|}{\textbf{43.04}} & \multicolumn{1}{c|}{\textbf{42.77}} & \textbf{43.59} & \multicolumn{1}{c|}{\textbf{43.51}} & \multicolumn{1}{c|}{\textbf{42.54}} &  \multicolumn{1}{c}{\textbf{42.90}}             \\ \hline
FedDyn                  & \multicolumn{1}{c|}{44.68}          & \multicolumn{1}{c|}{47.56}          & 47.87          & \multicolumn{1}{c|}{45.18}          & \multicolumn{1}{c|}{48.60}          & \multicolumn{1}{c}{48.74}                     \\ 
FedDyn+ASD (Ours)              & \multicolumn{1}{c|}{\textbf{47.51}} & \multicolumn{1}{c|}{\textbf{49.03}} & \textbf{50.32} & \multicolumn{1}{c|}{\textbf{47.81}} & \multicolumn{1}{c|}{\textbf{50.23}}  & \multicolumn{1}{c}{\textbf{51.48}}            \\ \hline
\end{tabular}
}
\label{sup:fix_num_clnt_vary_cp}
}

\end{table}

\subsection{Hessian Analysis}
\label{sup:hess_analysis}

In the Table~\ref{sup:hess_tab} we analyze the top eigenvalue and the trace of the Hessian of the global model when ASD is applied to methods such as FEdProx, FedNTD, FedSAM and FedDisco. 
%We see that ASD consistently reduces the top eigenvalue and trace  

\begin{table}[htp]
\centering 
\caption{The table shows the impact of ASD on the algorithms on CIFAR-100 Dataset. We consistently see that the top eigenvalue and the trace of the Hessian decrease and the Accuracy improves when ASD is used. This suggests that using ASD makes the global model reach to a flat minimum for better generalization.}
\scalebox{0.8}{
\begin{tabular}{c|c|c|c}
\toprule
Algorithm      & Top Eigenvalue $\downarrow$ & Trace $\downarrow$  & Accuracy $\uparrow$  \\ \midrule

FedProx        & 45.2            & 8683  & 37.79    \\
FedProx + ASD  & \textbf{11.9}            & \textbf{2663}  & \textbf{41.31}    \\\midrule

FedNTD         & \textbf{16.3}            & 3517  & 40.40    \\ 
FedNTD + ASD   & 17.5            & \textbf{2840}  & \textbf{43.01}    \\ \midrule
FedSAM         & 19.04            & 4022  & 40.89    \\ 
FedSAM + ASD   & \textbf{6.0}            & \textbf{1339}  & \textbf{43.99}    \\ \midrule
FedDisco       & 46.7            & 8771  & 38.97    \\ 
FedDisco + ASD & \textbf{12.2}            & \textbf{2334}  & \textbf{41.55}    \\ \midrule
\end{tabular}
}
\label{sup:hess_tab}
\end{table}

\subsection{Performance on CIFAR-10 dataset}

In Table~\ref{cifar10_table}, we show the results for the CIFAR-10 dataset,  we find that applying the ASD improves the performance of all the algorithms consistently.
\label{sup:cifar10}
\begin{table}[htp]
\centering
\caption{We show the accuracy attained by the algorithms on CIFAR-10 at the end of 500 communication rounds. It can be seen that by combining the proposed approach the performance of all the algorithms is improved.}
\scalebox{0.8}{
\begin{tabular}{c|c|c|c}
\hline
Algorithm  & $\delta = 0.3$          & $\delta = 0.6$           & iid          \\ \toprule
FedAvg              & 78.15 $\pm{0.78}$          & 78.66 $\pm{0.10}$           & 80.99 $\pm{0.09}$  \\
FedAvg+ASD (\textbf{Ours})          & \textbf{79.01} $\pm{0.33}$ & \textbf{79.93} $\pm{0.21}$ & \textbf{81.83} $\pm{0.19}$ \\ \midrule
FedProx              & 78.25 $\pm{0.68}$          & 78.81 $\pm{0.69}$           & 81.04 $\pm{0.34}$  \\
FedProx+ASD (\textbf{Ours})          & \textbf{78.77} $\pm{0.49}$ & \textbf{79.91} $\pm{0.12}$ & \textbf{81.74} $\pm{0.06}$ \\ \midrule
FedNTD              & 76.79 $\pm{0.37}$          & 78.55 $\pm{0.31}$           & 80.98 $\pm{0.21}$ \\ 
FedNTD+ASD (\textbf{Ours})          & \textbf{78.78} $\pm{0.86}$ & \textbf{80.13} $\pm{0.49}$ & \textbf{81.80} $\pm{0.11}$  \\ \midrule
FedDyn              & 81.08 $\pm{0.52}$         & 81.48 $\pm{0.35}$          & 83.51 $\pm{0.27}$  \\ 
FedDyn+ASD (\textbf{Ours})          & \textbf{81.82} $\pm{0.56}$ & \textbf{82.33} $\pm{0.39}$  & \textbf{84.09} $\pm{0.15}$  \\ \midrule
FedDisco              & 78.21 $\pm{0.45}$         & 78.76 $\pm{0.32}$          & 81.04 $\pm{0.30}$  \\ 
FedDisco+ASD (\textbf{Ours})          & \textbf{78.97} $\pm{0.01}$ & \textbf{79.98} $\pm{0.35}$  & \textbf{81.71} $\pm{0.21}$  \\ \midrule
FedSpeed              & 81.28 $\pm{0.32}$         & 81.83 $\pm{0.36}$          & 83.67 $\pm{0.14}$  \\ 
FedSpeed+ASD (\textbf{Ours})          & \textbf{81.70} $\pm{0.20}$ & \textbf{82.62} $\pm{0.26}$  & \textbf{84.57} $\pm{0.24}$  \\ \bottomrule

\end{tabular}
}
\label{cifar10_table}
\end{table}

\subsection{Accuracy vs Communication rounds}
\label{sup:acc_comm}
In the below figures~\ref{sup:fig_delta3}~\ref{sup:fig_delta6} and ~\ref{sup:fig_delta_iid}, we present how the accuracy is evolving across the communication rounds for the FL methods FedNTD, FedProx, FedDisco with and without the  ASD regularizer. We present these results for non-iid  ($\delta = 0.3$ and $\delta=0.6)$ and with the iid data partitions for both the CIFAR-100 and Tiny-ImageNet datasets. It can be seen that adding ASD to these off-the-shelf FL methods consistently improves the performance.

\begin{figure*}[htp]
  \centering
    \subfigure[CIFAR-100 ($\delta = 0.3$)]{\includegraphics[scale=0.3]{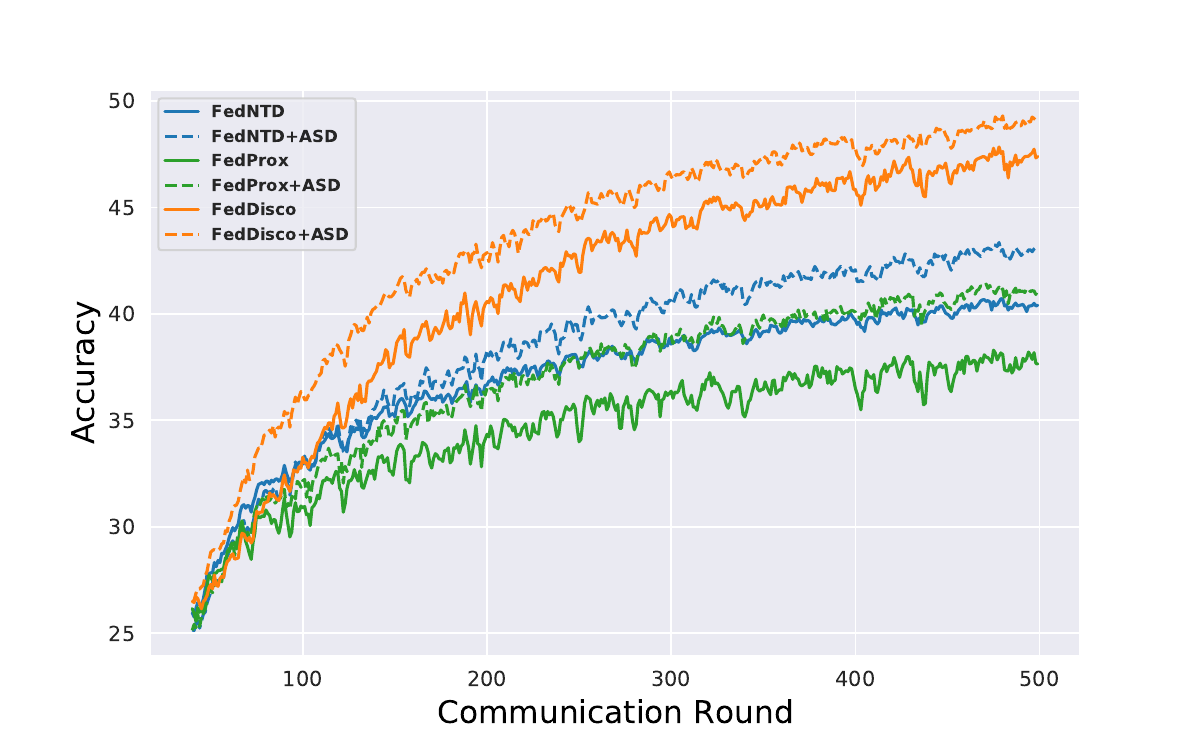}}\hspace{-2.2em}
  \subfigure[Tiny-ImageNet ($\delta = 0.3$)]{\includegraphics[scale=0.3]{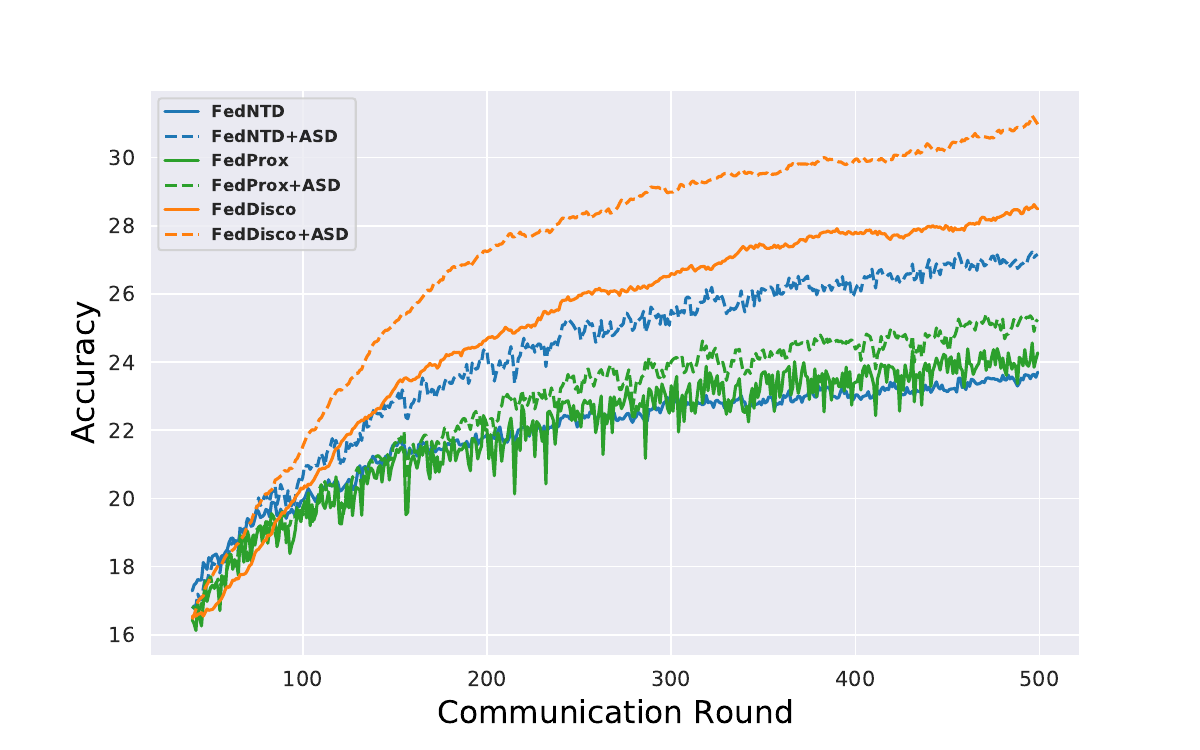}}\hspace{-2.2em}
  %%%\vspace{-0.1in}
  \caption{Test Accuracy vs Communication rounds: Comparison of algorithms with $\delta = 0.3$, data partition on CIFAR-100 and Tiny-ImageNet datasets. All the algorithms augmented with proposed regularization (ASD) outperform compared to their original form.}
  \label{sup:fig_delta3}
\end{figure*}

\begin{figure*}[htp]
  \centering
  \subfigure[CIFAR-100 ($\delta = 0.6$)]{\includegraphics[scale=0.3]{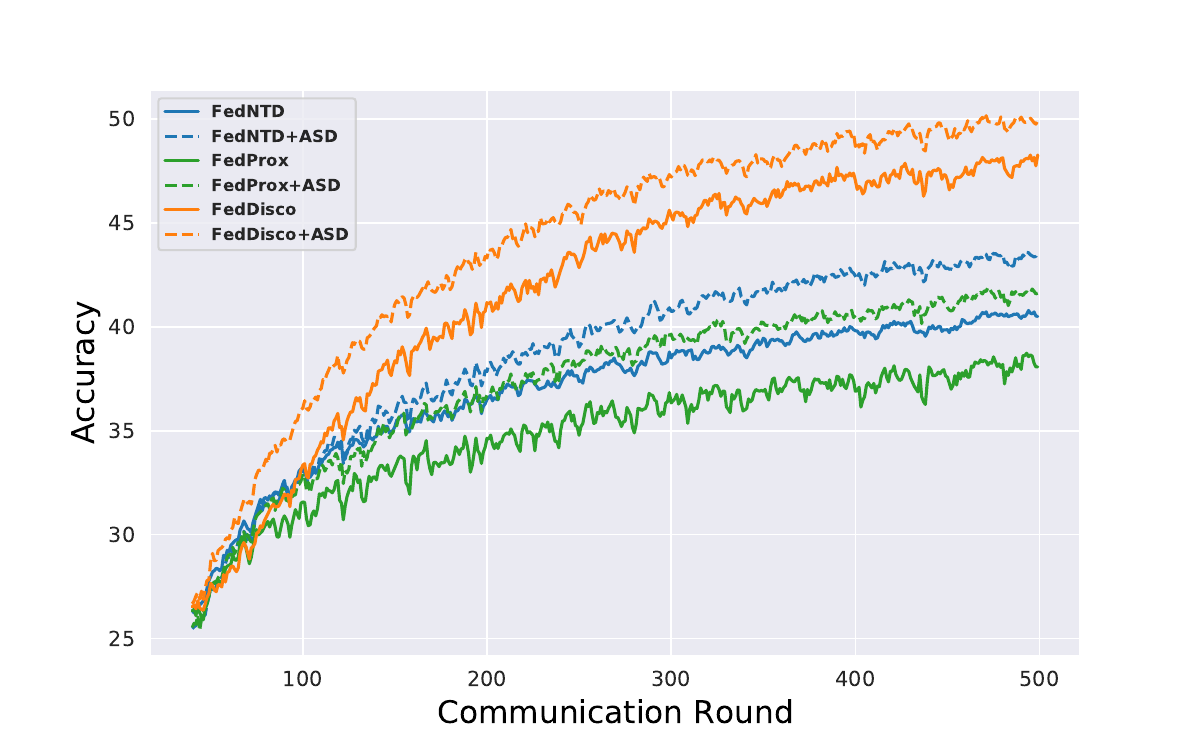}}\hspace{-2.2em}
  \subfigure[Tiny-ImageNet ($\delta = 0.6$)]{\includegraphics[scale=0.3]{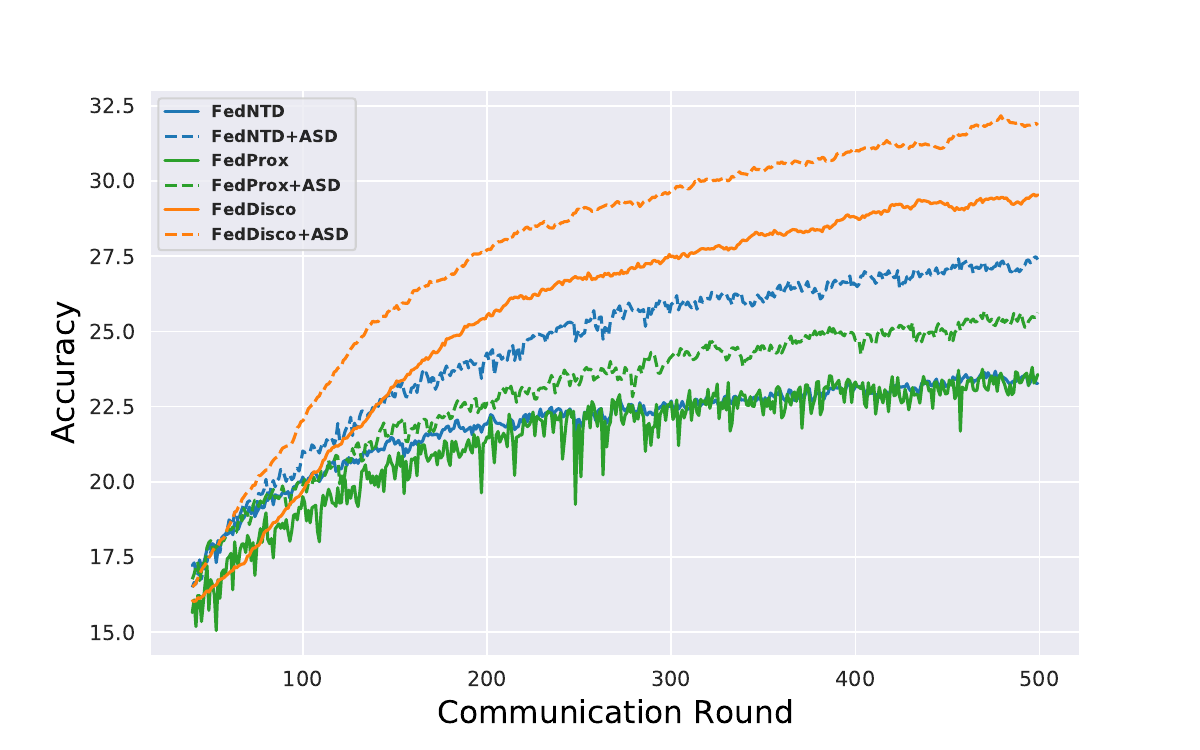}}\hspace{-2.2em}
 
\vspace{-0.1in}
  \caption{Test Accuracy vs Communication rounds: Comparison of algorithms with $\delta = 0.6$ data partition on CIFAR-100 and Tiny-ImageNet datasets. All the algorithms augmented with proposed regularization (ASD) outperform compared to their original form. %FedDyn+ASD outperforms all the other algorithms.
  }
  \label{sup:fig_delta6}
\end{figure*}
\vspace{-0.1in}
\begin{figure*}[htp]
  \centering
  \subfigure[CIFAR-100 (iid)]{\includegraphics[scale=0.3]{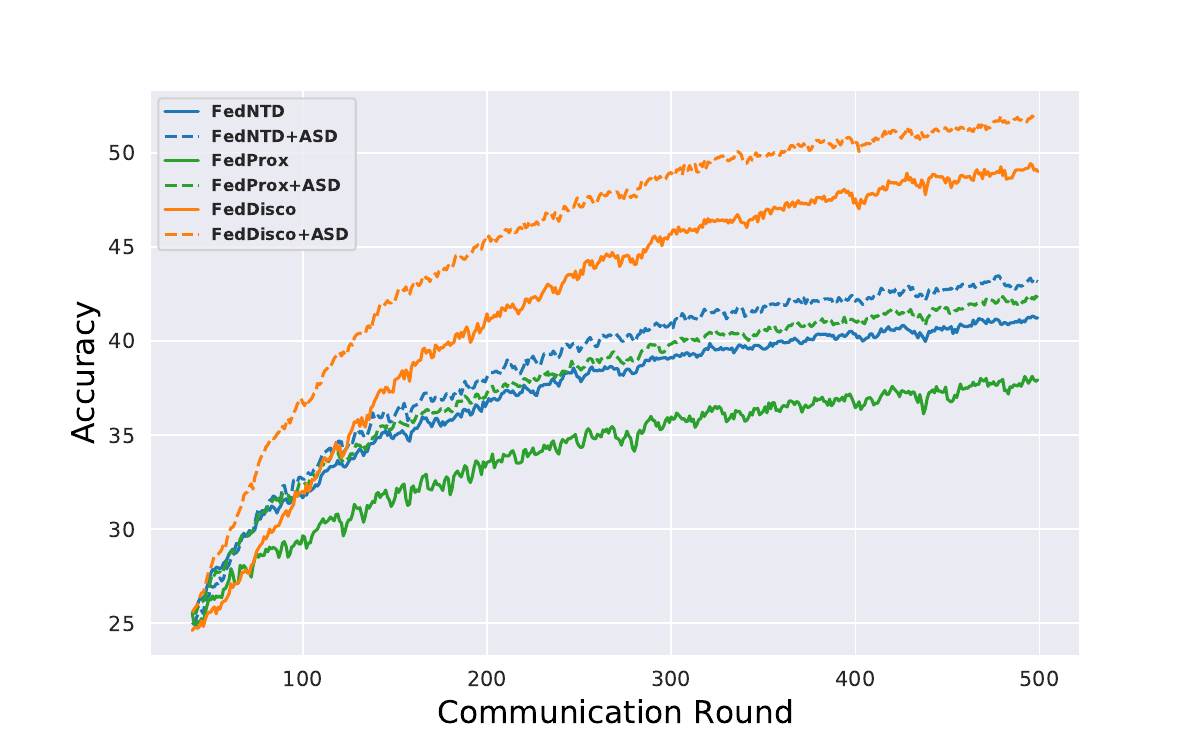}} \hspace{-2.2em}
  \subfigure[Tiny-ImageNet (iid)]{\includegraphics[scale=0.3]{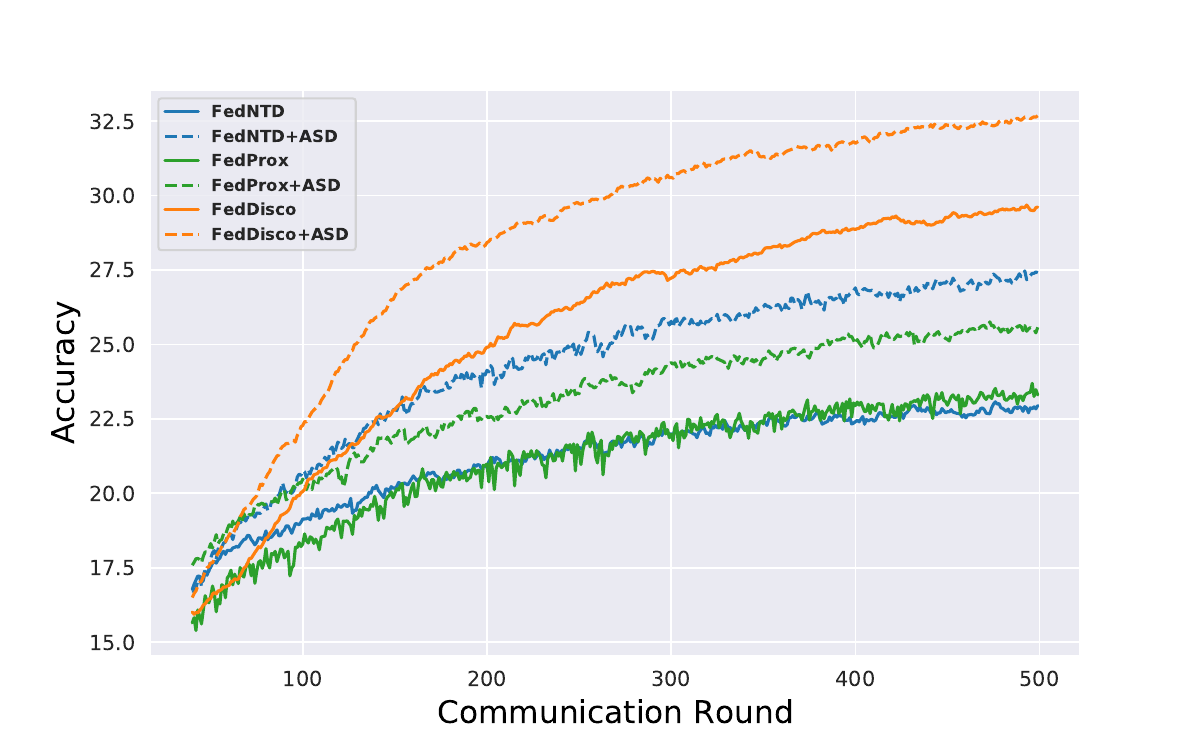}} \hspace{-2.2em}
  %%%\vspace{-0.1in}
  \caption{Test Accuracy vs Communication rounds: Comparison of algorithms with iid data partitions on CIFAR-100 and Tiny-ImageNet datasets. All the algorithms augmented with proposed regularization (ASD) outperform compared to their original form. %FedDyn+ASD outperforms all the other algorithms.
  }
  \label{sup:fig_delta_iid}
\end{figure*}

\subsection{Privacy of Proposed Method}
In our method, which is ASD regularizer, the adaptive weights are computed by the client without depending  on the server and it does not assume access to any auxiliary data at the server as assumed in methods such as FedCAD~\citep{he2022class} and FedDF~\citep{lin2020ensemble}. In our method,  only model parameters are communicated with the server similar to FedAvg~\citep{mcmahan2017communication}. Thus our privacy is similar to the FedAvg method at the same time obtaining significant improvements in the performance.

\subsection{Implementation of ASD with the FL Methods}.
\label{asd_plus_algs}
We now present the integration of ASD loss with the existing FL methods. For all the methods FedAvg, FedDyn, FedSpeed, FedProx, FedDisco and FedSAM, we augment the client loss of each of these methods with our proposed ASD loss in the Eq~\ref{eq7}.
FedNTD~\citep{lee2021preservation} uses the non-true distillation loss, it distills the knowledge only from the non-true classes.
\begin{equation}
\mathcal{D}_{\text{NTD}}(q_{g}(x^i)||q_{k}(x^i)) = \sum_{c \neq y}^{C}q_{g}^{c}(x^i)log(q_{g}^{c}(x^i)/q_{k}^{c}(x^i))
\label{eq_ntd}
\end{equation}
The above equation represents the FedNTD loss on the sample $i$, when the true class label is $y$. 
We now use the adaptive weights as defined in Eq.~\ref{sup:alpha_eq}, to update the FedNTD loss as below.

\begin {equation}
L_{k}^{asd-ntd}(\mathbf{w})  \triangleq \sum_{i \in [B]}\alpha_{i}^{k} \mathcal{D}_{\text{NTD}}(q_{g}(x^i)||q_{k}(x^i))
\label{ntd_loss}
\end{equation}
So the final loss used for optimizing FedNTD with adaptive self-distillation is given below.
\begin{equation}
f_{k}(\mathbf{w}) \triangleq L_{k}(\mathbf{w}) + \lambda L_{k}^{asd-ntd}(\mathbf{w})
\label{final_ntd_loss}
\end{equation}

where $ L_{k}(\mathbf{w})$ is defined as in Eq.~\ref{exp_erm} of the main paper.

{\subsection{On the choice of KL divergence}
\label{kl_choice}
}
%\textcolor{mycolor}
{
 The distillation loss introduced by the seminal work of~\citep{hinton2015distilling} matches the temperature-raised softmax values between the pre-trained teacher model and student model for effective knowledge transfer. It is essentially cross entropy between two softmax vectors. KL divergence differs from cross entropy by a constant and hence achieves the same optimization objective. In our context, we treat the server model as the teacher model and the client model as the student model. Other divergence measures such as reverse KL and JS can also be considered, but we did not see any significant performance difference empirically. In fact KL divergence performed better compared to reverse-KL and JS divergence as shown in Table below. For this experiment we used the CIFAR-100 dataset with 100 clients and 10\% client participation rate.}

\begin{table}[htp]
\centering
%\textcolor{mycolor}
{
\centering
\caption{Comparison of Statistical Divergences}
\scalebox{0.75}{
\begin{tabular}{c|c}
\toprule
Method       & Accuracy (in \%) \\ \hline
KL (ours)       & \textbf{42.77}            \\ 
reverse-KL    & 42.04   \\ 
Jensen-Shannon       & 42.21            \\  \bottomrule
\end{tabular}
}
}
\label{kl_choice_tab}
\end{table}
\begin{figure}[htp]
  \centering
  \includegraphics[width=0.43\linewidth]{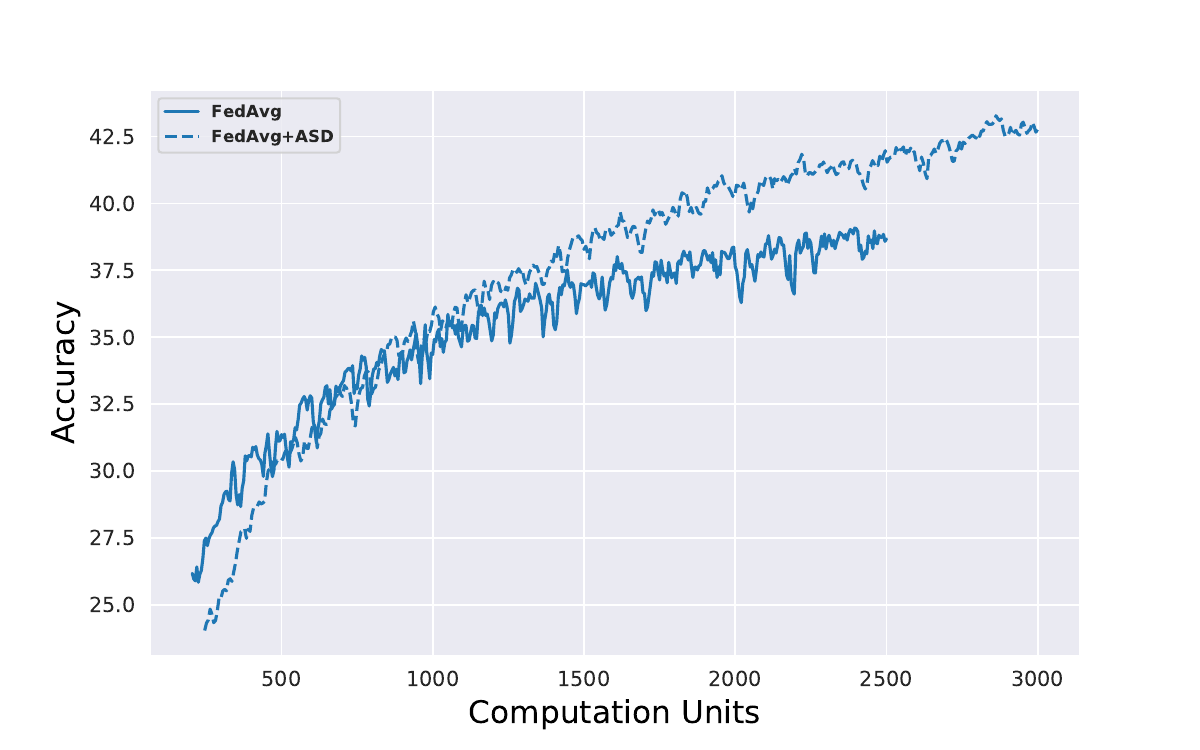}
  \textcolor{mycolor}{\caption{Comparison of the Communication vs Computation for FedAvg and FedAvg+ASD. It can be seen that over the communication rounds for a given amount of computation, FedAvg+ASD attains better accuracy compared to FedAvg.}
  \label{fig_comm_compute}}
\end{figure}
%\textcolor{mycolor}
\subsection{Computation vs Accuracy}
\label{comp_vs_acc}
%\textcolor{mycolor}
{In the Figure~\ref{fig_comm_compute} we have compared the computation with the accuracy for FedAvg and FedAvg+ASD methods. In particular, we observe that at a fixed cumulative computation cost of 2500 units FedAvg attains $38.67$ \% Accuracy while FedAvg+ASD attains $42.3$\% accuracy. Here the one unit denotes the computation required for the single forward pass.}

\subsection{Proofs of Propositions:}
\label{sup:proof}
We rewrite the adaptive weighting equations for convenience as below.

\begin{equation}
\alpha_{k}^{i}= \frac{\hat{\alpha_{k}}^{i}} {\sum_{i \in B}\hat{\alpha_{k}}^{i}}
\label{sup:alpha_eq}
\end{equation}
and $\hat{\alpha_{k}}^{i}$ is defined as below Eq.~\ref{sup:alpha_hat_eq}
\begin{equation}
\hat{\alpha_{k}}^{i} \triangleq {exp(-\mathcal{H}(x^i)) \over p_{k}^{y^{i}}}
\label{sup:alpha_hat_eq}
\end{equation}

\iffalse
\begin{equation}
\alpha_{i}^{k}= \frac{\hat{\alpha_{i}}^{k}} {\sum_{i \in B}\hat{\alpha_{i}}^{k}}
\label{sup:alpha_eq}
\end{equation}
\fi 

%%\vspace{0.1in}
\begin{proposition}
 $\inf_{\mathbf{w}\in \mathbb{R}^d} {G_d(\mathbf{w},\lambda)}$ is $1$, $\forall$ $\lambda$
\end{proposition}

\begin{proof}
    
\begin{equation}
G_d = {{{1 \over K }\sum_{k} {\lVert \nabla f_k  \rVert}^2} \over {{\lVert \nabla f \rVert}^2}}    
\end{equation}
\begin{equation}
{{\lVert \nabla f \rVert}^2} = {{\lVert {1 \over K}{\sum_k{\nabla{f_k}}}  \rVert}^2} 
\end{equation}
%\textcolor{mycolor}
{We observe that the function ${{\lVert . \rVert}^2} $ is Convex. By applying Jensen's inequality, we get the desired result. The expectation is taken over the discrete probability measure.}  
\begin{equation}
{{\lVert \nabla f \rVert}^2} \leq  {1 \over K}\sum_k{{\lVert {{\nabla{f_k}}}  \rVert}^2} 
\end{equation}
\end{proof}
%\vspace{0.1in}
\begin{lemma}
For any function of the form $ \zeta(x) = \frac{ax^2+bx+c_n}{ax^2+bx+c_d}$ satisfying $c_n > c_d$ , $\exists$ $x_c \geq 0$ such that $\frac{d\zeta(x)}{dx} < 0$ $\forall$ $x \geq x_c$
\label{lem1}
\end{lemma}
\begin{proof}
\begin{dmath}
\frac{d\zeta(x)}{dx} = \frac{(2ax+b)(ax^2+bx+c_d) - (2ax+b)(ax^2+bx+c_n) }{(ax^2+bx+c_d)^2}   
\end{dmath}
By re-arranging and simplifying the above we get the following

\begin{equation}
\frac{d\zeta(x)}{dx} = \frac{2x(ac_d-ac_n) + b(c_d - c_n)}{(a_dx^2+bx+c_d)^2}   
\end{equation}

%Lets us call the numerator in the above equation as $\eta(x) = 2x(ac_d-ac_n) + b(c_d - c_n)$. \\
We are interested in knowing when the numerator is negative.\\

\begin{equation}
x2a(c_d-c_n) \leq  {b(c_n - c_d)}
\end{equation}

Since $c_n > c_d$, we have 
\begin{equation}
x2a(c_n-c_d) \geq  {-b(c_n - c_d)} \implies  x \geq \frac{-b}{2a}
\end{equation}

assuming $x_c = \lvert \frac{-b}{2a} \rvert$

We have the desired condition for $x \geq x_c$
%The value of $\frac{-b}{2a} is > 0$ as $b$ is assumed to be negative and $a > 0$.

\iffalse
We now analyze the two cases $b < 0$ and $b > 0$ separately.\\
For $(b > 0)$
\\
As $\eta(x)$ is a simple quadratic in x, we observe that the coefficient of $x^2$, i.e., $b(a_n - a_d) < 0$ by assumptions, and the discriminant $\Delta$  is given as.
\begin{equation}
\Delta =  4(a_nc_d-a_dc_n)^2 - 4 b^2 (a_n-a_d)(c_d-c_n)
\label{disc_eq}
\end{equation}
We observe that $\Delta > 0$ as $(a_n-a_d)(c_d-c_n) < 0$ by assumptions. This implies that $\eta(x)$ has two real roots and the fact that $\eta(0) > 0$, there must exist a positive root $x_c$ such that $\eta(x) < 0$  whenever $x > x_c$. It is also worth noticing that that $x_c$ increases with b.
\fi 
This concludes the proof.

%\textbf{Case 2:} $(b < 0)$\\
% We observe that the coefficient of $x^2$ i.e, $b(a_n - a_d) > 0$ by assumptions, and the discriminant $\Delta$  is given in Eq.~\ref{disc_eq}.
% We observe that $\Delta > 0$. This implies that $\eta(x)$ has two real roots and the fact that $\eta(0) < 0$, there must exists a positive root $x_c$ such that $\eta(x) < 0$  whenever $x < x_c$.\\ This concludes the proof.
\end{proof}

\begin{proposition}
When the class conditional distribution across the clients is identical, i.e., $\mathbb{P}_{k}(x \mid y) = \mathbb{P}(x \mid y)$  then $\nabla{f_{k}(\mathbf{w})} = \sum_c{p_k^c}(\mathbf{g}_{c} + \lambda  {\gamma}_{k}^{c} \tilde{\mathbf{g}}_c )$, where  $\mathbf{g}_c = \nabla{\mathbb{E}[{l(\mathbf{w};x,y)}\mid{y=c}]}$, and $\tilde{\mathbf{g}}_c = \nabla{\mathbb{E}[{ \exp({-\mathcal{H}(x)}) \mathcal{D}_{\text{KL}}(q_{g}(x)||q_{k}(x))} \mid{y=c} ]}$ where ${\gamma}_{k}^{c} = \frac{1}{p_k^c}$. 
\end{proposition}

\begin{proof}
We re-write the equations for $f_k(\textbf{w})$ from Sec 3.3 of main paper, $L_k(\textbf{w})$ and $L_k^{ASD}(\textbf{w})$ from the Sec 3.2 of main paper for convenience.

\begin{equation}
f_{k}(\mathbf{w}) =  L_{k}(\mathbf{w}) + \lambda L_{k}^{ASD}(\mathbf{w})
\label{sup:th_eq1}
\end{equation}

\begin{equation}
L_k(\mathbf{w}) = \underset{x,y \in D_k}{\mathbb{E}}[l_{k}(\mathbf{w};(x,y))]
\label{sup:exp_erm}
\end{equation} 

%\begin {equation}
%L_{k}^{ASD}(\mathbf{w})  \triangleq \sum_{i \in [B]}\alpha_{i}^{k} $\mathcal{D}_{\text{KL}}(q_{g}^{i}||q_{k}^{i})
%\label{sup:eq7}
%\end{equation}

\begin {equation}
L_{k}^{ASD}(\mathbf{w})  \triangleq \mathbb{E}[\alpha_{k}(x,y) \mathcal{D}_{\text{KL}}(q_{g}(x)||q_{k}(x))]
\label{sup:eq7}
\end{equation}

By applying the tower property of expectation, we expand Eq.~\ref{sup:exp_erm} as below

\begin{equation}
L_{k}(\mathbf{w}) = \sum_c{p_k^c}\mathbb{E}[{l_{k}(\mathbf{w};x,y)}\mid{y=c}]
\label{th_eq2}
\end{equation}

If we assume the class-conditional distribution across the clients to be identical the value of $\mathbb{E}[{l_{k}(\mathbf{w};x,y)}\mid{y=c}]$ is same for all the clients. Under such assumptions, we can drop the client index $k$ and rewrite the Eq.~\ref{th_eq2} as follows
\begin{equation}
L_{k}(\mathbf{w}) = \sum_c{p_k^c}\mathbb{E}[{l(\mathbf{w};x,y)}\mid {y=c}]
\label{th_eq4}
\end{equation}
\begin{equation}
\nabla{L_{k}(\mathbf{w})} = \sum_c{p_k^c}\nabla{\mathbb{E}[{l(\mathbf{w};x,y)} \mid {y=c}]}
\label{th_eq5}
\end{equation}

We further simplify the notation by denoting $\mathbf{g}_c = \nabla{\mathbb{E}[{l(\mathbf{w};x,y)} \mid {y=c}]}$.  
\begin{equation}
\nabla{L_{k}(\textbf{w})} = \sum_c{p_k^c}\mathbf{g}_c
\label{th_eq6}
\end{equation}

To make the analysis tractable, In Eq.~\ref{sup:alpha_hat_eq}, we use the un-normalized weighting scheme as the constant can be absorbed into $\lambda$. we can re-write Eq.~\ref{sup:alpha_hat_eq} as below

\begin{equation}
\hat{\alpha}_k^{i} = \gamma_k^{y} \exp({\mathcal{H}(x)})
\end{equation}
where $\gamma_k^{y} = \frac{1}{p_k^y}$
With the above assumptions we can interpret the Eq.~\ref{sup:eq7} as follows.
\begin{equation}
L_{k}^{ASD} = \underset{x,y \in D_k}{\mathbb{E}}[l_{k}^{dist}(\mathbf{w};(x,y))]
\end{equation}
where $l_{k}^{dist}(\mathbf{w};(x,y)) = \gamma_{k}^{y}\exp(-{\mathcal{H}(x)}) \mathcal{D}_{\text{KL}}(q_{g}(x)||q_{k}(x))$.

By following the similar line of arguments from Eq.~\ref{th_eq2} to Eq.~\ref{th_eq5} we can write the following  
\begin{equation}
\nabla{L_{k}^{ASD}(\mathbf{w})} = \sum_c{p_k^c}\tilde{\mathbf{g}}_c \gamma_k^c
\label{th_eq8}
\end{equation}

\begin{equation}
\nabla{f_{k}(\mathbf{w})} = \sum_c{p_k^c}(\mathbf{g}_{c} + \lambda  \gamma_k^c \tilde{\mathbf{g}}_c )
\label{th_eq9_proof}
\end{equation}

\end{proof}

\begin{lemma}
If $c_n = {K} \sum_{k=1}^{K}\sum_{c=1}^{C}({p_k^c})^2$, $c_d = \sum_{k1=1}^{K}\sum_{k2=1}^{K}\sum_{c=1}^{C}({p_{k1}^c} {p_{k2}^c}) $. 
%$a_n = KC$ and $a_d = K^2C$. 
where $p_k^c \geq 0$ $\forall k,c$, and $\sum_{c=1}^{C} p_k^c = 1$, then $\frac{c_n}{c_d} \geq 1 $. 
\label{lem2}
\end{lemma}
\begin{proof}
We need to show that 
\begin{equation}
\frac{\sum_{k=1}^{K}\sum_{c=1}^{C}({p_k^c})^2}{\sum_{k1=1}^{K}\sum_{k2=1}^{K}\sum_{c=1}^{C}({p_{k1}^c} {p_{k2}^c})} \geq \frac{1}{K}
\end{equation}

By rewriting the denominator we get
\begin{equation}
{\frac{\sum_{k=1}^{K}\sum_{c=1}^{C}({p_k^c})^2}{\sum_{c=1}^{C}(\sum_{k=1}^{K}({p_{k}^c}))^2} \geq \frac{1}{K}}
\end{equation}

Consider rewriting the denominator of the L.H.S of above equation.
\begin{equation}
{\sum_{c=1}^{C}(\sum_{k=1}^{K}({p_{k}^c}))^2} = \sum_{c=1}^{C}({(\mathbf{p^c})^\intercal} {\textbf{1}})^2
\label{lem2_eq1}
\end{equation}

where $\mathbf{p^c} = [p_1^c p_2^c ... p_K^c]^\intercal$ and $\textbf{1}$ is the all one vector of size $K$

Applying the Cauchy Schwartz inequality to the R.H.S of the Eq.~\ref{lem2_eq1} we get the following.

\begin{equation}
\sum_{c=1}^{C}({(\mathbf{p^c})^\intercal} {\textbf{1}})^2 \leq \sum_{c=1}^{C}\sum_{k=1}^{K} ({p_k^c})^2 K
\label{lem2_eq2}
\end{equation}
By combining the Eq.~\ref{lem2_eq1} and Eq.~\ref{lem2_eq2} the result follows.
\end{proof}

\begin{proposition}
When the class-conditional distribution across the clients is the same, and the Assumption~\ref{assump1} holds then $\exists$ a range of values for $\lambda$ such that whenever $\lambda \geq \lambda_{c}$ we have $\frac{dG_d}{d\lambda} < 0$ and $G_d(\mathbf{w},\lambda) < G_d(\mathbf{w},0)$.
\label{sup:th1}
\end{proposition}

\begin{proof}
\begin{equation}
G_d = {{{1 \over K } \sum_{k = 1}^{K} {\lVert \nabla f_k  \rVert}^2} \over {{\lVert \nabla f \rVert}^2}}
\label{sup:gd_def}
\end{equation}

From Sec 3.2 of the main paper we have the following, We drop the argument $\mathbf{w}$ for the functions $f_k$ to simplify the notation

\begin{equation}
\nabla{f_{k}} = \sum_{c = 1}^{C}{p_k^c}(g_{c} + \lambda  {\gamma}_k^{c} \tilde{g}_c )
\label{th_eq9_grad}
\end{equation}

\begin{dmath}
{\lVert \nabla{f_{k}} \rVert}^2  = \sum_{c1 = 1}^{C}\sum_{c2 = 1}^{C}{p_k^{c1} p_k^{c2}}(\mathbf{g}_{c1}^\intercal + \lambda  {\gamma_k^{c1}}\tilde{\mathbf{g}}_{c1} ^\intercal )(\mathbf{g}_{c2} + \lambda  {\gamma_k^{c2}}\tilde{\mathbf{g}}_{c2} )\notag \\
= \sum_{c1 = 1}^{C}\sum_{c2 = 1}^{C}{p_k^{c1} p_k^{c2}}(\mathbf{g}_{c1}^\intercal\mathbf{g}_{c2} +  
\lambda  \gamma_{k}^{c2}\mathbf{g}_{c1}^\intercal \tilde{\mathbf{g}}_{c2} +
\lambda  {\gamma_k^{c1}}\tilde{\mathbf{g}}_{c1}^\intercal \mathbf{g}_{c2} + \lambda^2{\gamma_k^{c2}}{\gamma_k^{c1}}\tilde{\mathbf{g}}_{c1}^\intercal \tilde{\mathbf{g}}_{c2}  )\notag \\
%~\approx \sum_{c = 1}^{C}(p_k^{c})^2(\mathbf{g}_{c}^\intercal\mathbf{g}_{c} +  
%\lambda  {\gamma_k^{c}}\mathbf{g}_{c}^\intercal \tilde{\mathbf{g}}_{c} +
%\lambda  {\gamma_k^{c}}\tilde{\mathbf{g}}_{c}^\intercal \mathbf{g}_{c} + %\lambda^2{\gamma_k^{c}} {\gamma_k^{c}} \tilde{\mathbf{g}}_{c}^\intercal \tilde{\mathbf{g}}_{c}  )\notag \\
~\approx \sum_{c = 1}^{C}(p_k^{c})^2(\mathbf{g}_{c}^\intercal\mathbf{g}_{c}) + 
\lambda  \sum_{c1 = 1}^{C}\sum_{c2 = 1}^{C} {p_k^{c1}}\mathbf{g}_{c1}^\intercal \tilde{\mathbf{g}}_{c2} +
\lambda  \sum_{c1 = 1}^{C}\sum_{c2 = 1}^{C}{p_k^{c2}}\tilde{\mathbf{g}}_{c1}^\intercal \mathbf{g}_{c2} + \lambda^2\sum_{c = 1}^{C}(p_k^{c})^2{\gamma_k^{c}} {\gamma_k^{c}} \tilde{\mathbf{g}}_{c}^\intercal \tilde{\mathbf{g}}_{c}  \notag \\
%= \sum_{c = 1}^{C}(p_k^{c})^2(1 +  
%2\lambda  {\gamma_k^{c}}\mathbf{g}_{c}^\intercal \tilde{\mathbf{g}}_{c} + %\lambda^2({\gamma_k^{c}})^2)\notag \\
= \sum_{c = 1}^{C}(p_k^{c})^2 +  
2\lambda  \sum_{c1 = 1}^{C}\sum_{c2 = 1}^{C} {p_k^{c1}}\mathbf{g}_{c1}^\intercal \tilde{\mathbf{g}}_{c2}
+ \lambda^2 C
\label{th_eq9}
\end{dmath}

In the above equation the second equality is obtained by simply expanding the product, the third approximation by weakly correlated assumption of the gradients. The last two equalities used the fact that  ${\gamma}_k^{c} = \frac{1}{p_k^c}$. We also assume that gradients are normalized to unit magnitude. 

Finally, we have the following
\begin{dmath}
\frac{1}{K}\sum_{k=1}^{K}{\lVert \nabla{f_k} \rVert}^2 = \frac{1}{K} (\sum_{k=1}^{K}\sum_{c=1}^{C}({p_k^c})^2 + 
2\lambda  \sum_{k=1}^{K}\sum_{c1=1}^{C} \sum_{c2=1}^{C}{p}_{k}^{c1}\mathbf{g}_{c1}^\intercal \tilde{\mathbf{g}}_{c2} + \lambda^2KC 
) \notag \\
= \frac{1}{K^2}(a_n\lambda^2 + b_n\lambda + c_n)
\label{num_eq}
\end{dmath}
where
\begin{equation}
a_n \coloneqq K^2C
\label{an_def}
\end{equation}
\begin{equation}
b_n  \coloneqq  2K\sum_{k=1}^{K}\sum_{c1=1}^{C} \sum_{c2=1}^{C}{p}_{k}^{c1}\mathbf{g}_{c1}^\intercal \tilde{\mathbf{g}}_{c2} \notag \\
\label{b_def}
\end{equation}
\begin{equation}
c_n \coloneqq K\sum_{k=1}^{K}\sum_{c=1}^{C}({p_k^c})^2
\label{cn_def}
\end{equation}
\begin{dmath}
{\lVert \nabla{f} \rVert}^2 = (\lVert {\frac{1}{K}} \sum_{k=1}^{K} \sum_{c=1}^{C} {p_k^c}(\mathbf{g}_{c} + \lambda  {\gamma}_k^{c} \tilde{\mathbf{g}}_c ) \rVert)^2 \notag \\
= {\frac{1}{K^2}} \sum_{k1=1}^{K} \sum_{k2=1}^{K} \sum_{c1=1}^{C} \sum_{c2=1}^{C} {p_{k1}^{c1}}(\mathbf{g}_{c1}^\intercal + \lambda  \gamma_{k1}^{c1} \tilde{\mathbf{g}}_{c1}^\intercal )
 {p_{k2}^{c2}}(\mathbf{g}_{c2} + \lambda  {\gamma}_{k2}^{c2} \tilde{\mathbf{g}}_{c2} ) \notag \\
%\approx {\frac{1}{K^2}} \sum_{k1=1}^{K} \sum_{k2=1}^{K} \sum_{c=1}^{C} {p_{k1}^{c}}{p_{k2}^{c}}(\mathbf{g}_{c}^\intercal\mathbf{g}_{c} +  
%\lambda  {\gamma_{k1}^{c}}\mathbf{g}_{c}^\intercal \tilde{\mathbf{g}}_{c} +
%\lambda  {\gamma_{k2}^{c}}\tilde{\mathbf{g}}_{c}^\intercal \mathbf{g}_{c} + 
%\lambda^2{\gamma_{k1}^{c}} {\gamma_{k2}^{c}} \tilde{\mathbf{g}}_{c}^\intercal \tilde{\mathbf{g}}_{c}) \notag \\
\approx {\frac{1}{K^2}} \sum_{k1=1}^{K} \sum_{k2=1}^{K} \left(\sum_{c=1}^{C} {p_{k1}^{c}}{p_{k2}^{c}}(\mathbf{g}_{c}^\intercal\mathbf{g}_{c}) +  
 \sum_{c1=1}^{C} \sum_{c2=1}^{C} \left( \lambda   {{p_{k1}^{c1}}{p_{k2}^{c2}}\gamma_{k2}^{c2}}\mathbf{g}_{c1}^\intercal \tilde{\mathbf{g}}_{c2} +
\lambda {p_{k1}^{c1}} {p_{k2}^{c2}}{\gamma_{k2}^{c}}\tilde{\mathbf{g}}_{c1}^\intercal \mathbf{g}_{c2} + 
\lambda^2{{p_{k1}^{c1}} {p_{k2}^{c2}}\gamma_{k1}^{c}} {\gamma_{k2}^{c}} \tilde{\mathbf{g}}_{c}^\intercal \tilde{\mathbf{g}}_{c}\right)\right) \notag \\
= {\frac{1}{K^2}} \sum_{k1=1}^{K} \sum_{k2=1}^{K} \sum_{c=1}^{C} ({p_{k1}^{c}}{p_{k2}^{c}}) +  
\lambda  \sum_{k1=1}^{K} \sum_{k2=1}^{K} \sum_{c1=1}^{C} \sum_{c2=1}^{C}  {p}_{k1}^{c1}\mathbf{g}_{c1}^\intercal \tilde{\mathbf{g}}_{c2} +
\lambda  \sum_{k1=1}^{K} \sum_{k2=1}^{K} \sum_{c1=1}^{C} \sum_{c2=1}^{C} p_{k2}^{c2}\tilde{\mathbf{g}}_{c1}^\intercal \mathbf{g}_{c2} + \lambda^2 K^2C\notag \\
= {\frac{1}{K^2}} (\sum_{k1=1}^{K} \sum_{k2=1}^{K} \sum_{c=1}^{C} ({p_{k1}^{c}}{p_{k2}^{c}}) +  
2\lambda K\sum_{k=1}^{K} \sum_{c1=1}^{C} \sum_{c2=1}^{C}  {p}_{k}^{c1}\mathbf{g}_{c1}^\intercal \tilde{\mathbf{g}}_{c2} +
 \lambda^2 K^2 C)
\label{den_eq} \notag \\
= {\frac{1}{K^2}}(a_d\lambda^2 + b\lambda + c_d )
\end{dmath}

By defining
\begin{equation}
a_d \coloneqq K^2C
\label{ad_def}
\end{equation}

\begin{equation}
b_d  \coloneqq  2K\sum_{k=1}^{K}\sum_{c1=1}^{C} \sum_{c2=1}^{C}{p}_{k}^{c1}\mathbf{g}_{c1}^\intercal \tilde{\mathbf{g}}_{c2} \notag \\
\label{bd_def}
\end{equation}

\begin{equation}
c_d \coloneqq \sum_{k1=1}^{K}\sum_{k2=1}^{K}\sum_{c=1}^{C}({p_{k1}^c} {p_{k2}^c})    
\end{equation}
\label{cd_def}

By substituting  Eq.~\ref{num_eq} and Eq.~\ref{den_eq} in Eq.~\ref{sup:gd_def} we get

\begin{equation}
G_d(\mathbf{w},\lambda) = \frac{a_n\lambda^2 + b_n\lambda + c_n}{a_d\lambda^2 + b_d\lambda + c_d}
\label{Gd_w_def}
\end{equation}
Comparing Eq.~\ref{an_def} and Eq.~\ref{ad_def} we see that 
$a \coloneqq a_n = a_d $, $b \coloneqq b_n = b_d $.\\

%By assumption, we have $b < 0$. 
Also $c_n > c_d$ assuming $p_k^c$ is non-degenerate.\\
%Similarly by comparing Eq.~\ref{cn_def} and Eq.~\ref{cd_def} we see that $c_n < c_d$.
%From Lemma~\ref{lem2} We have $\frac{c_n}{c_d} \geq 1 $.

%We are interested in the values of $\lambda$ for which $G_d(\mathbf{w},\lambda) < G_d(\mathbf{w},0)$. Solving for $\lambda$ we get the following inequality
%\begin{equation}
%\lambda > \frac{b(c_d - c_n)}{a_dc_n - a_nc_d }    
%\end{equation}

Using the Lemma~\ref{lem1} on Eq.~\ref{Gd_w_def} we get the value of $\lambda_b$ such that $G_d(\mathbf{w},\lambda)$ is reduced.   \\ 

We also get $\lambda \geq \lvert \frac{-b}{a} \rvert$ by analyzing the values of $\lambda$ for which 
$G_d(\mathbf{w},\lambda) < G_d(\mathbf{w},0) $ holds.\\

Thus choosing the $\lambda > \lambda_c  = \sup_{b: -k^2C \leq b \leq k^2C} {max{(\lambda_b, \lvert \frac{-b}{a} \rvert)}}$ guarantees 
 $G_d(\mathbf{w},\lambda) < G_d(\mathbf{w},0) $, for all $\mathbf{w}$.\\

%We can attain the value of $G_d(\textbf{w},\lambda) = 1$ by increasing $\lambda$ to infinity. This however has the down-side as very high values of $\lambda$ restrict the local learning which results in very small gradient updates which is undesired. 

%\textbf{Case2:} ($b < 0$).
%Again Applying the Lemma~\ref{lem1} to $\eta(\lambda)$ when $b<0$ we get the value of $\lambda_c$ such that whenever $\lambda < \lambda_c$ the value of $G_d(\mathbf{w},\lambda) < G_d(\mathbf{w},0)$ and $G_d(\mathbf{w},\lambda)$ decreases when $\lambda \in (0,\lambda_c)$.   
This concludes the proof.%\footnote{The value of $b$ can be negative at the beginning of the training as the training progresses the value of $b$ will be positive as the contribution of the true class gradient will become higher. In practice, one could set the value of $\lambda$ smaller at the beginning of the training and gradually increase across the communication rounds. In our implementations, we only set a single value of $\lambda$ to simplify the tuning effort. }

\end{proof}

\subsection{Discussion on Impact of Gradient Dissimilarity on the Convergence}
\label{sup:convg_disc}

We now study how the gradient diversity impacts the convergence of the FL algorithms such as FedProx and FedAvg. We omit the dependence of $\lambda$ on $B$. (for these algorithms $\lambda =0$ so $B$ is nothing but $B(0)$ in our notation)
We have the gradient dissimilarity assumption below
\begin{assumption}
${{{1 \over K }\sum_{k} {\lVert \nabla f_k(\mathbf{w})  \rVert}^2} \leq  B^2{{\lVert \nabla f(\mathbf{w}) \rVert}^2}}$
\label{sup:assump2}
\end{assumption}

\subsubsection{FedProx}
Suppose the functions $f_k$ are lipschiltz smooth and their exists $L_- > 0$ such that $\mathbf{H}{f_k} \succeq L_{-} \mathbf{I}$. With $\bar{\mu} - L_{} > 0$, where $\mu$ is FedProx regularization. If $f_k$ satisfies the assumption~\ref{sup:assump2} then acccording to Theorem $6$ of~\citep{li2020federated} the FedProx, after $T = {\mathcal{O}( \frac{\Delta}  {\rho \epsilon} )} $. We have the gradient contraction  as $\frac{1}{T} {\sum_{t = 0}^{T-1} \mathbf{E}{\lVert f(\mathbf{w}^t) \rVert}^2} \leq \epsilon $. The value of $\rho$ is given below.\\
\begin{equation}
 \rho = \frac{1}{\mu} -\frac{\gamma B}{\mu} - \frac{B(1+\gamma)\sqrt(2)}{\bar{\mu} \mu} -\frac{L(1+\gamma)^2B^2}{2{\bar{\mu}}^2} - \frac{L(1+\gamma)^2B^2}{K{\bar{\mu}}^2}(2\sqrt{2K} + 2) > 0
 \label{p_eq}
 \end{equation}
 for some $\gamma > 0$ and $\Delta = f(\mathbf{w}^0)-f(\mathbf{w}^*) $, $f(\mathbf{w}^*$ is the local minimum.\\
 It can be seen that the convergence is inversely related to $\rho$. High value of $\rho$ leads to faster convergence. From Eq.~\ref{p_eq} we can see that $\rho$ can be increased by decreasing the value of $B$. Thus reducing the value of $B$ helps in better convergence.
\subsubsection{FedAvg}

\begin{assumption}
We now analyze the convergence of FedAvg, we consider the following assumptions
$\lVert \nabla f_k(\mathbf{x}) - f_k(\mathbf{x}) \rVert = \beta \lVert \mathbf{x} -  \mathbf{y} \rVert$ ($\beta$ smoothness)
\label{sup:assump3}
\end{assumption}

\begin{assumption}
Gradients have bounded Variance.
\label{sup:assump4}
\end{assumption}

%. 
Suppose that $f(\mathbf{w})$ and $f_k(\mathbf{w})$, satisfies Assumptions~\ref{sup:assump2}, ~\ref{sup:assump3} and ~\ref{sup:assump4}. Let $\mathbf{w}^* = \underset{\mathbf{w}}{\arg\min} \ f(\mathbf{w}) $ the local step-size be $\alpha_l$. The theorem \MakeUppercase{\romannumeral 5} in~\citep{karimireddy2020scaffold} shows that FedAvg algorithm will have contracting gradients. If Initial model is $\mathbf{w}^0$, $F = f(\mathbf{w}^0)-f(\mathbf{w}^*)$ and for constant $M$, then in $R$ rounds, the model $w^R$ satisfies
$\mathbb{E}[{\lVert \nabla{f(\mathbf{w}^R)} \rVert}^2] \leq 
{O({{\beta M \sqrt{F}} \over {\sqrt{RLS}} } + {{\beta B^2F} \over {R}})
}$.

We see the convergence rate is ${O({{\beta M \sqrt{F}} \over {\sqrt{RLS}} } + 
{{\beta B^2F} \over {R}})}$. We can see that convergence has a direct dependence on $B^2$. This is the only term that is linked to heterogeneity assumption. So the lower value of $B$ implies faster convergence. This motivates to have a tighter bound on heterogeneity. ASD achieves this by introducing the regularizer and choosing the appropriate value of $\lambda$. In the figure~\ref{sup:fig_prox_avg} we empirically we verify the impact of ASD on the convergence. We plot the smoothed estimates of the norm of the difference of the global model parameters between the successive communication rounds i.e $\lVert \mathbf{w}^t - \mathbf{w}^{t-1} \rVert$. 

\label{gd_conv}

\begin{figure*}[htp]
  \centering
  \subfigure[FedAvg]{\includegraphics[scale=0.40]{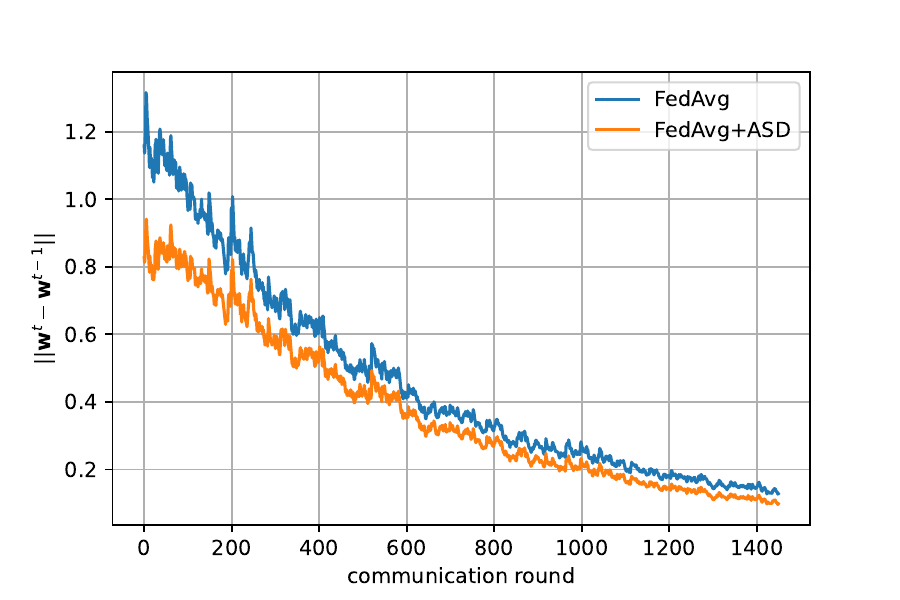}}
  \subfigure[FedProx]{\includegraphics[scale=0.40]{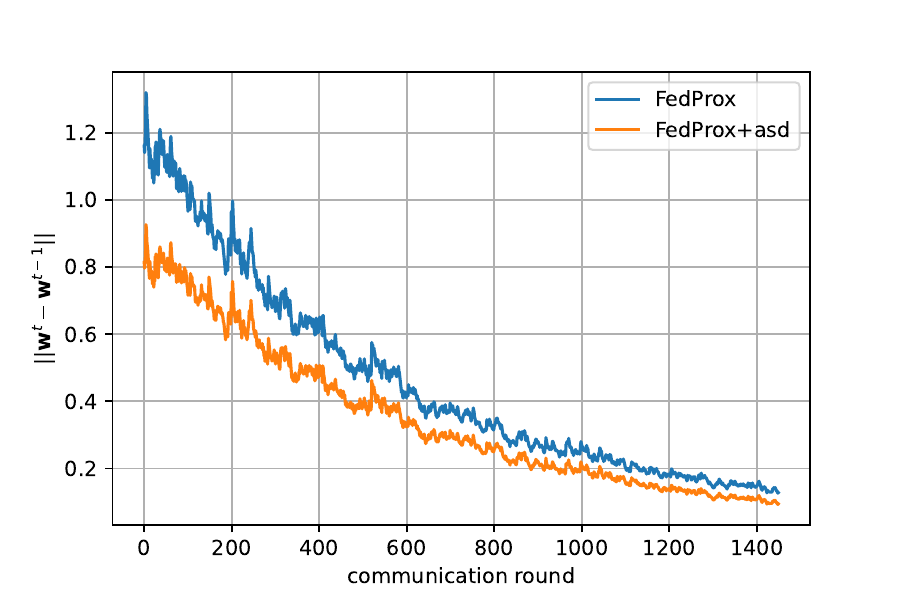}}
  %\vspace{-0.12in}
  \caption{Impact of ASD on the convergence on CIFAR-100 dataset with non-iid partition of $\delta = 0.3$}
  \label{sup:fig_prox_avg}
\end{figure*}

\end{document}